\def\year{2021}\relax
\documentclass[letterpaper]{article} 
\usepackage{aaai21}  
\usepackage{times}  
\usepackage{helvet} 
\usepackage{courier}  
\usepackage[hyphens]{url}  
\usepackage{graphicx} 
\usepackage{natbib}  
\usepackage{caption} 
\usepackage{amsthm}
\newtheorem{theorem}{Theorem}

\theoremstyle{remark}

\theoremstyle{definition}

\usepackage{bm}
\usepackage{amsmath}
\usepackage{amstext}
\usepackage{subcaption}
\usepackage[switch]{lineno}

\title{Scalable Graph Networks for Particle Simulations}
\author {
        Karolis Martinkus,\textsuperscript{\rm 1}
        Aurelien Lucchi, \textsuperscript{\rm 1}
        Nathanaël Perraudin \textsuperscript{\rm 2} \\
}
\affiliations {
    \textsuperscript{\rm 1} ETH Zürich \\
    \textsuperscript{\rm 2} Swiss Data Science Center \\
    martinkus@ethz.ch, aurelien.lucchi@inf.ethz.ch, nathanael.perraudin@sdsc.ethz.ch
}

\begin{document}
\frenchspacing

\maketitle

\begin{abstract}
Learning system dynamics directly from observations is a promising direction in machine learning due to its potential to significantly enhance our ability to understand physical systems. 
However, the dynamics of many real-world systems are challenging to learn due to the presence of nonlinear potentials and a number of interactions that scales quadratically with the number of particles $N$, as in the case of the N-body problem. In this work, we introduce an approach that transforms a fully-connected interaction graph into a hierarchical one which reduces the number of edges to $O(N)$. This results in linear time and space complexity while the pre-computation of the hierarchical graph requires $O(N\log (N))$ time and $O(N)$ space. Using our approach, we are able to train models on much larger particle counts, even on a single GPU. We evaluate how the phase space position accuracy and energy conservation depend on the number of simulated particles. Our approach retains high accuracy and efficiency even on large-scale gravitational N-body simulations which are impossible to run on a single machine if a fully-connected graph is used. Similar results are also observed when simulating Coulomb interactions. Furthermore, we make several important observations regarding the performance of this new hierarchical model, including: i) its accuracy tends to improve with the number of particles in the simulation and ii) its generalisation to unseen particle counts is also much better than for models that use all $O(N^2)$ interactions.
\end{abstract}

\section{Introduction}
\noindent The ability to simulate complex systems is invaluable to many fields of science and engineering. Constructing simulators for such systems by hand can be very labour intensive or even impossible if the underlying processes are not understood or no sufficiently precise and accurate approximations of the interactions are known. To address this, various data-driven methods for learning system dynamics have been investigated \cite{battaglia2016interaction,  mrowca2018flexible, li2018learning,  greydanus2019hamiltonian, sanchezgonzalez2019hamiltonian, sanchezgonzalez2020learning, finzi2020generalizing}. However, they either use only local interactions between close-by particles or they simulate systems with just a few particles. Unfortunately, this does not capture many real-world scenarios, where systems are comprised of thousands of particles that interact over long distances. Using only local interactions would cause very high errors in such cases.

In this contribution, we focus in particular on the N-body problem, since it cannot be solved to sufficient accuracy with only local information. The nonlinearity of the interactions and the fact that the system is chaotic if the number of particles $N > 2$ \cite{roy2012predictability} makes this problem particularly complicated. This complexity can be seen in the original work on Hamiltonian Neural Networks \cite{greydanus2019hamiltonian} where the performance on the 2-body problem was good, but significantly deteriorated on the 3-body problem.

We propose a hierarchical model\footnote{Code available at: \url{https://git.io/JtUXt}} which builds on top of existing graph network (GN) architectures such as \cite{battaglia2018relational} and previous work on accurate physical simulations of a few particles \cite{sanchezgonzalez2019hamiltonian}. Our hierarchical architecture is physically motivated by the multipole expansion and inspired by the fast multipole method (FMM) \cite{greengard1988rapid}. Our method allows us to extend existing models and simulate complex systems that require $O(N^2)$ interactions with thousands of particles which we have observed empirically to be infeasible when a fully connected graph is used. Importantly, models that use our hierarchical approach retain similar accuracy to models working with a fully connected graph and a smaller particle count.

Finally, we note that the fast multipole method and multipole expansion have been applied to various other problems, such as flow simulations \cite{koumoutsakos1995high}, acoustics \cite{gunda2008boundary}, molecular dynamics \cite{board1992accelerated, ding1992atomic} and even interpolation of scattered data (reconstruction of a 3D object mesh) \cite{carr2001reconstruction}. This suggests that our hierarchical method should also facilitate learning on a similarly wide range of problems.

\section{Related Work}
Recent studies show that neural networks can successfully learn to simulate complex physical processes \cite{battaglia2016interaction, sanchez2018graph, mrowca2018flexible, li2018learning,  greydanus2019hamiltonian, sanchezgonzalez2019hamiltonian, sanchezgonzalez2020learning}. Most of the existing work focuses on introducing better physical biases such as a relational model \cite{battaglia2016interaction},  conservation law bias \cite{greydanus2019hamiltonian}, combining these with an ODE bias \cite{sanchezgonzalez2019hamiltonian} and various similar refinements  \cite{chen2019symplectic, zhong2019symplectic, saemundsson2020variational, desai2020vign}. Unfortunately, all of these highly accurate and noise resistant models are only able to work with tens of particles in complex spring or gravitational systems.
Models which are focused on simulating fluids, rigid and deformable bodies are able to work with thousands of particles \cite{mrowca2018flexible, li2018learning, sanchezgonzalez2020learning}. However, in these scenarios, it is possible to achieve state of the art performance by just using local information \cite{sanchezgonzalez2020learning}. 

There is a long history of attempts to improve the scalability of traditional particle simulations.
One of the first efficient methods was the tree code introduced by \cite{barnes1986hierarchical}. It uses a hierarchical spatial tree (i.e. quadtree) to define localised groups of particles. Particles then interact directly with other groups (instead of their member particles) if the distance to the group is sufficiently larger than the radius encompassing all of the particles in that group, to avoid large errors. Interaction list for each particle would be computed by going through the tree top-down and at each level adding all of the cells that are sufficiently far away and haven't been covered by previous interactions to the list.
At the lowest level, particles from neighbouring cells would interact directly. This algorithm has $O(N\log(N))$ complexity. Fast multipole method (FMM) \cite{greengard1987fast} improved this approach by adding cell-cell interactions and then propagating the resulting effects to the children. This algorithm has $O(N)$ complexity for the force estimation. It has been shown by \cite{dehnen2014fast} that by carefully tweaking the implementation of the FMM it is possible to achieve comparable gravitational force errors to a direct summation method that computes all $O(N^2)$ interactions.

In this work, we apply the lessons learned in scaling traditional particle simulations to current approaches of learning simulations from data using graph neural networks.

There have been attempts to incorporate a hierarchical structure to neural simulators of fluids, rigid and deformable bodies \cite{mrowca2018flexible, li2018learning}. However, in those cases, the focus was to improve the accuracy and effect propagation in simulations with only local interactions between particles. The resulting methods are not as scalable and are ill-suited for simulations of long-range force field interactions. \cite{li2018learning} restricted hierarchy to one level of groups. \cite{mrowca2018flexible} included edges between children and all of their grandparents in the hierarchy. This resulted in $O(N \log N)$ complexity when propagating information through their hierarchical graph. To construct the hierarchy, recursive applications of k-means clustering were used, which can be more computationally expensive than using a quadtree or an octree \cite{arthur2006slow}. In these hierarchies, nodes only interact with other nodes in the same group. This can result in two particles that are next to each other having no direct interaction, even though close interactions are the most important ones. We also know from physics that cluster-cluster force field interactions can only be approximated if the clusters are sufficiently far away \cite{dehnen2014fast}. Nevertheless, neighbouring clusters often interact directly in these hierarchies. These issues make existing hierarchical approaches ill-suited for force field interactions.
This is unsurprising, as they were designed for a fundamentally different problem.

Successful non-graph-based neural simulators have also been proposed \cite{ummenhofer2019lagrangian, finzi2020generalizing}. However, they only use local information from the particle's neighbourhood to update its position.

\section{Model}
\begin{figure}[h]
    \centering
    \includegraphics[width=1.0\columnwidth]{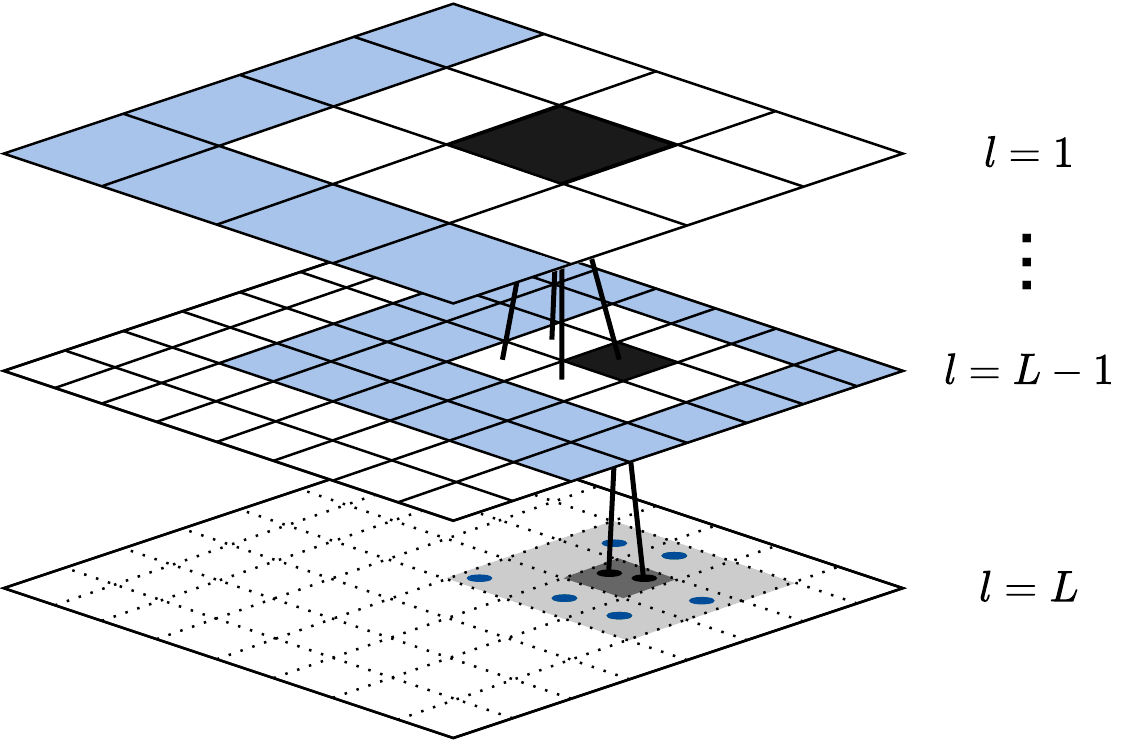}
    \caption{3 level hierarchy ($L=3$) from the point of view of the black particles that belong to the black cells. The first level ($l=1$) has $16$ cells, in each subsequent level, the number of cells quadruples. At the lowest level, we have the particles. Each cell is connected to its four child cells at the lower level. Cells at the second-lowest level ($L-1$) are instead connected to the particles that belong to them. First, during an upward pass cell features are recursively computed from the features of their children. Then, during the downward pass, at each level, the black cell interacts with its near-neighbours (blue). Aggregated interactions are propagated to the cell's children. At the lowest level ($L$), particles interact directly with particles (blue) from the neighbouring cells (grey).}
    \label{fig:hierarchy}
\end{figure}
\subsection{Hierarchical Graph}
We build the hierarchy by recursively subdividing the space into four parts (2D space is assumed, but our approach naturally generalises to 3D spaces as it is based on quadtrees). When a cell is split we add edges between it and its children. We assume that the particles are roughly uniformly distributed and repeat the splitting $\lfloor\log(N)\rceil$ times. In expectation, this results in having one particle per cell at the lowest level. In the general case, when the data is not uniform, each cell can be split until there are at most $k$ particles in it, where $k\geq1$ is some constant. For real-world data, this should also result in linear experimental complexity as it does for the FMM \cite{dehnen2014fast, kabadshow2012periodic}. Empty cells are pruned. We also include another level at the bottom of the hierarchy that holds all of the particles. Each particle has an edge to the cell it belongs to at the level above. Both, the particles and the cells are represented as nodes in the hierarchical graph. Particle nodes have mass, position and velocity as their features. For the cell nodes, the total mass, centre of mass position and velocity are pre-computed from their children. If Coulomb interactions are simulated, the particle charge and total charge features are added to the particle and cell nodes respectively.

Each cell is also connected to its near-neighbours. Near-neighbours are other cells at the same level that are not directly adjacent to the cell, but whose parents were adjacent to the cell's parent (blue, Figure \ref{fig:hierarchy}). This way we recursively push down the closest interactions to lower levels, where the interactions are more granular. Ensuring that interacting cells are never next to each other is also necessary in order to avoid potentially large errors \cite{dehnen2014fast}. At the lowest level particles are directly connected with other particles that belong to the same cell or the neighbouring cells. Importantly, two particles that are next to each other are never separated. This hierarchy also ensures that the receptive field of each particle is close to being symmetric at every level. The top-level with $4$ cells is removed from the hierarchy, as all of the cells are neighbours and do not interact. A visual depiction of the hierarchy can be seen in Figure \ref{fig:hierarchy}. 

It is easy to see that constructing a quadtree with $\log_4 N$ levels takes $O(N\log N)$ time and $O(N)$ space.
Next, we demonstrate that the strategy used to construct the hierarchy yields a number of nodes and edges that scale linearly with $N$ (instead of the typical quadratic complexity discussed earlier).

\begin{theorem}
  There are $O(N)$ nodes in the hierarchy.
\label{theorem:nodes}\end{theorem}
\begin{proof}
Based on our assumption that particles are uniformly distributed, after $\log_4(N)$ splits we will have one particle per cell and thus $N$ cells. In each level going from the bottom-up we will have $4$ times fewer cells. This gives us the following geometric progression for the total number of cells:
\begin{linenomath*}
\begin{equation*}
    N + \frac{N}{4} + ... + 16 = \sum^{\log_4(N)-2}_{l=0} N \left(\frac{1}{4}\right)^l < 2 N   \; .
\end{equation*}
\end{linenomath*}
Considering that in the hierarchy we also include one level with all of the particles, we will have $< 3 N = O(N)$ nodes in the hierarchical graph.
\end{proof}

\begin{theorem}
  There are $O(N)$ edges in the hierarchy and each node has at most a constant number of edges.
\label{theorem:edges}\end{theorem}
\begin{proof}
As we assume that particles are uniformly distributed and that in expectation we have $1$ particle per cell in the lowest level, each particle will have $8+1$ edges in expectation (since each particle is only connected to particles from its $8$ neighbouring cells and their parent cell).
 
Each cell in the hierarchy is connected to its parent ($\leq 1$), its children ($\leq 4$) and its near neighbours ($\leq 27$). There are at most $27$ near neighbours because the cell is connected to other cells that belong to its parent or the neighbours of its parent ($1+8$ parent cells, each with $4$ children - $36$ cells), but not itself or the cells that are its immediate neighbours ($36 - 1 - 8 = 27$).
 
So, in expectation, we will always have at most a constant number of edges per node. As we have $O(N)$ nodes as per Theorem \ref{theorem:nodes}, we will also have $O(N)$ edges.
\end{proof}

\subsection{Graph Networks}
Graph network (GN) models \cite{battaglia2018relational} operate on a graph $G=(V,E,\bm{u})$, with global features $\bm{u}$, nodes $V$ and edges $E$. Graph networks enforce a structure similar to traditional simulation methods, where first interactions (edges) between particles (nodes) are computed. Then all incoming interactions (edges) are aggregated per particle (node) and together with particle features are used to compute updated particle features. Global values such as the Hamiltonian (total energy) of the system can also be computed from the interactions and the particle features. This relational bias has been shown to greatly improve the accuracy of simulators compared to using a single multi-layer perceptron (MLP) \cite{battaglia2016interaction, sanchez2018graph}.

In the basic case, a particle system is represented as a fully connected graph, where each node is a particle. The node features we use are mass $m$, position $\bm{q}$, velocity $\dot{\bm{q}}$ and if applicable charge $z$. The edge matrix $E$ holds sender and receiver IDs for each edge. In our case relative node positions are used in the models. Meaning that one of the edge features used during the forward pass is the distance vector between the sender and the receiver. Node positions are masked everywhere else.

As graph network models perform distinct operations on each edge and each vertex their time and space complexity is linear in the number of edges and the number of vertices. If a fully connected graph is used this results in overall $O(N^2)$ time and space complexity. 

\subsubsection{Delta Graph Network (DeltaGN).}
DeltaGN is analogous to previous direct neural simulators \cite{battaglia2016interaction, sanchez2018graph, sanchezgonzalez2019hamiltonian, sanchezgonzalez2020learning}. This model directly predicts the change in particle position $\bm{q}$ and velocity $\dot{\bm{q}}$:
\begin{linenomath*}
\begin{equation*}
(\bm{q},\dot{\bm{q}})_{t+1} = (\bm{q},\dot{\bm{q}})_{t} + (\Delta \bm{q}, \Delta \dot{\bm{q}}) \; ,
\end{equation*}
\end{linenomath*}
where $(\Delta \bm{q}, \Delta \dot{\bm{q}}) = \text{GN}_{V}(V,E,\Delta t)$ are the new vertex features produced by the graph network.

\subsubsection{Hamiltonian ODE Graph Network (HOGN).}
HOGN \cite{sanchezgonzalez2019hamiltonian} uses a graph network to compute the Hamiltonian of the system (single scalar):
\begin{linenomath*}
\begin{equation*}
H_{\text{GN}}(\bm{q}, \bm{p}=m\cdot \dot{\bm{q}}) = \text{GN}_u(V,E) \; .
\end{equation*}
\end{linenomath*}
By deriving this output w.r.t. the inputs of the network (particle position $\bm{q}$ and momentum $\bm{p}$) and using Hamilton's equations we can recover the derivatives of particle position and momentum:
\begin{linenomath*}
\begin{equation*}
f^{HOGN}_{\dot{\bm{q}}, \dot{\bm{p}}} = \left( \frac{\partial H}{\partial p}, - \frac{\partial H}{\partial q} \right) = (\dot{\bm{q}}, \dot{\bm{p}}) \; .
\end{equation*}
\end{linenomath*}
The position updates for a given $\Delta t$ are then produced by a differentiable Runge–Kutta 4 (RK4) integrator that repeatedly queries $f^{HOGN}_{\dot{\bm{q}}, \dot{\bm{p}}}$.

\subsection{Adapting Existing Models}

The idea behind adapting existing graph network models to use the hierarchy is simple: instead of using a dense representation of the particle interactions as input, we use the sparse particle interaction graph from the lowest level of the hierarchy (Figure \ref{fig:hierarchy}). When the graph network model constructs updated edge feature matrix $E'$ for this sparse graph, we append to it edges coming to each particle from its parent cell. This augmented edge feature matrix is then used to update vertex features and to calculate any global features in the same manner as done in the original models. 

The construction of this special edge -- representing all distant interactions for a particle -- requires propagating the information through the hierarchy. To this end, the updated lowest level cell feature vectors are built from their children as:
\begin{linenomath*}
\begin{equation*}
\bm{v}_c' = \bm{v}_c \oplus \sum_{p \in \text{children}(c)} \phi^{p\rightarrow c}(m_c, \dot{\bm{q}}_c, m_p, \dot{\bm{q}}_p, (\bm{q}_c - \bm{q}_p)) \; ,
\end{equation*}
\end{linenomath*}
where $v_c$ is the original feature vector of a cell (total  mass and centre of mass velocity), $p$ is the child particle, $\phi^{p\rightarrow c}$ is an MLP and $\oplus$ is a concatenation operator.
The features of the parent cell $c_p$ in the upper levels are built in a similar way from their children cells $c_c$:
\begin{linenomath*}
\begin{equation*}
\bm{v}_{c_p}' = \bm{v}_{c_p} \oplus \sum_{c_c \in \text{children}({c_p})} \phi^{{c_c}\rightarrow {c_p}}(m_{c_p}, \dot{\bm{q}}_{c_p}, \bm{v}_{c_c}', (\bm{q}_{c_p} - \bm{q}_{c_c})).
\end{equation*}
\end{linenomath*}
The $\phi^{{c_c}\rightarrow {c_p}}$ MLP uses a different set of parameters from $\phi^{p\rightarrow c}$, but $\phi^{{c_c}\rightarrow {c_p}}$ parameters are shared between all of the levels.

During the top-down pass at each level cell-cell interactions are computed:
\begin{linenomath*}
\begin{equation*}
\bm{e}_{c_j,c_i}' = \phi^{c\rightarrow c}(\bm{v}_{c_i}', \bm{v}_{c_j}', (\bm{q}_{c_i} - \bm{q}_{c_j})) \; .
\end{equation*}
\end{linenomath*}
For each cell, incoming interactions are aggregated (summed) together with interactions propagated from the parent:
\begin{linenomath*}
\begin{equation*}
\bm{e}_{\text{parent}_c}' = \phi^{c_p\rightarrow c_c}(\bm{v}_{c}', \bm{v}_{c_p}', (\bm{q}_{c} - \bm{q}_{c_p})) \; .
\end{equation*}
\end{linenomath*}
These aggregated interactions $\bm{e}_c' = \bm{e}_{\text{parent}_c}' + \sum_{c_j} \bm{e}_{c_j,c}'$ are used to update cell features:
\begin{linenomath*}
\begin{equation*}
\bm{v}_c'' = \bm{v}_{c} \oplus  \phi^{c}(\bm{v}_c', \bm{e}_c') \; .
\end{equation*}
\end{linenomath*}
Finally, the cell-particle edges are computed:
\begin{linenomath*}
\begin{equation*}
\bm{e}_{\text{parent}_p}' = \phi^{c\rightarrow p}(m_p, \dot{\bm{q}}_p, \bm{v}_c'', (\bm{q}_p - \bm{q}_c)) \; .
\end{equation*}
\end{linenomath*}
The $\phi^{c_p\rightarrow c_c}$ MLP uses different parameters from $\phi^{c\rightarrow p}$. Parameters of  $\phi^{c_p\rightarrow c_c}$ and $\phi^{c\rightarrow c}$ are shared between the hierarchy levels.

If any global parameters $\bm{u}$ are used by the model they are also used as part of the input to all of the hierarchy MLPs.

It is easy to see that all of these operations have time and space complexity that is linear in the number of nodes and edges.

Further information about the implementation of the models can be found in Appendix~\ref{appendix:models}.

\section{Experiments}
\subsection{Data}
We built our own N-body simulator. It uses a symplectic Leapfrog integrator that preserves the total energy of the system, Plummer force softening that helps avoid the singularity which arises when the distance between two particles goes to zero, individual and dynamic time steps for particles \cite{dehnen2011n}. The time step is set for each particle as a fraction of the base time step based on its acceleration. The particle state (mass, position, velocity) is saved at every base time step. The system we simulate uses periodic boundary conditions which means that when a particle leaves the unit cell its exact copy enters it on the opposite side.
Further information on the simulator can be found in Appendix~\ref{appendix:data}.

We simulate 1000 training, 200 validation and  200 test trajectories. Every trajectory is 200 base time steps long.
Particle positions are initialised uniformly at random inside the unit cell, mass is set to 1 and each velocity vector component is initialised uniformly at random over $(-1,1)$. Additionally, 
if we simulate Coulomb interactions each particle has its charge initialised uniformly at random over $(0.5,1.5)$ and assigned a random sign. The base time step is set to $\Delta t = 0.01$. Gravitational and Coulomb constants are respectively set to $G = 2$ and $k = 2$. We assume that in our simulated galaxy all of the units are dimensionless.

\subsection{Results}
\begin{figure*}[t]
  \centering
  \includegraphics[width=1.0\linewidth]{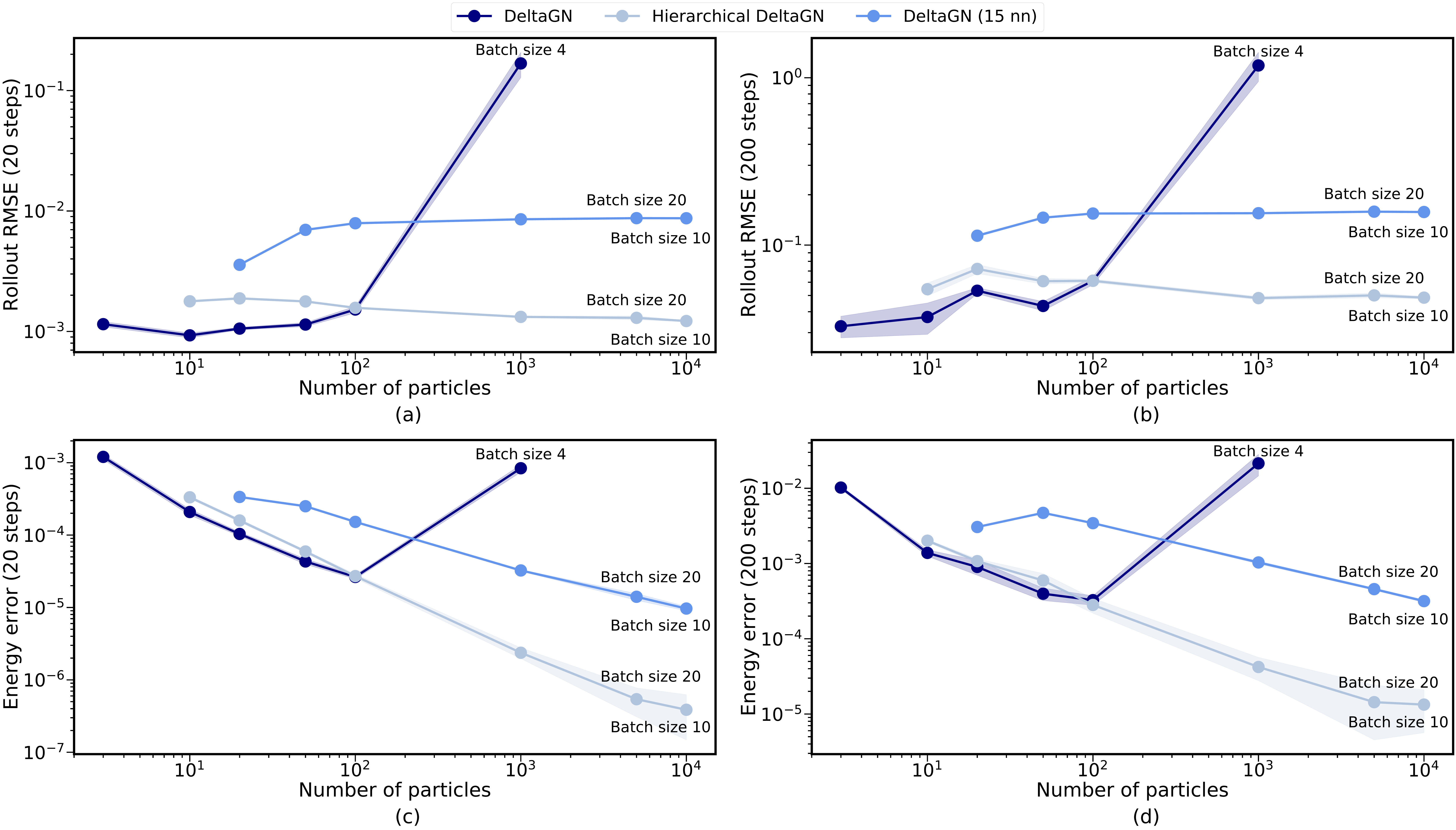}
  \caption{Models trained and evaluated on increasing particle counts. (a) 20 step rollout RMSE, (b) 200 step rollout RMSE, (c) energy error after 20 steps, (d) energy error after 200 steps. Batch sizes smaller than $100$ were used if not enough memory was available (shown in the graph).}
  \label{fig:batch_size_100}
\end{figure*}
We compare the model that uses our hierarchy (Hierarchical DeltaGN) against two baselines: model that uses a fully connected graph (DeltaGN) and a more computationally efficient model that uses $15$ nearest neighbour graph (DeltaGN (15 nn)). We also perform an experiment to test if our hierarchical approach is compatible with the HOGN model.
Models are trained for 500 thousand steps using a batch size of $100$ unless stated otherwise. We exponentially decay the learning rate every 200 thousand steps by $0.1$. The initial learning rate for all of the models was set to $0.0003$. The models are optimised using ADAM \cite{kingma2014adam}, and the mean square error (MSE) which is computed between the predicted and true phase space coordinates after one time step.

All datasets have the same particle density, meaning when increasing the particle count we increase the size of the unit cell accordingly.

When the rollout error is reported, it means that $200$ test trajectories $T$ are unrolled for a specified number of time steps by supplying the model with initial particle positions and then feeding the model its own outputs for the subsequent time steps. The input graph is rebuilt each time using the model's outputs. The error is computed as RMSE between phase-space coordinates of the predictions and the ground truth over the unrolled trajectory.

Gravitational and Coulomb systems are conservative, which means that their total energy should stay constant throughout their evolution. From this arises a common accuracy measure used in N-body simulations - relative energy error between the first and the last system states. We calculate mean relative energy error over all test trajectories:
\begin{linenomath*}
\begin{equation*}
    \text{Energy error} = \frac{1}{|T|} \sum_{i=1}^{|T|} \frac{H_{i, 0} - \hat{H}_{i, \tau}}{H_{i, 0}}   \; ,
\end{equation*}
\end{linenomath*}
where $\tau$ is the number of time steps the trajectory is unrolled for and $H_{i, t}$ is the Hamiltonian (total energy) of the system at time step $t$ of the trajectory $i$.

All errors are averaged over 5 independent runs of the model. We report the mean error and the standard deviation.

The experiments were performed on a machine with an Intel Xeon E5-2690 v3 CPU (2.60GHz, 12 cores, 24 threads), 64GB RAM and NVIDIA Tesla P100 GPU (16GB RAM).

\subsubsection{Scaling to Larger Particle Counts.}
From Figure \ref{fig:batch_size_100} we can see that hierarchical DeltaGN does not perform as well on small particle counts ($< 100$) where the hierarchy is small. However, it quickly catches up when $100$ particles are simulated and is able to simulate large particle counts with even better accuracy. Simulations with more particles are not feasible with DeltaGN that uses the fully-connected graph, as it soon runs out of memory and on the $1000$ particle dataset already suffers from terrible performance. This is probably due to the very large edge count per node. The DeltaGN (15 nn) model that uses only local connectivity is able to scale to large particle counts but suffers from poor accuracy due to the ignored long-range interactions. When the same small batch size is used for all of the models the situation stays similar (Figure \ref{fig:batch_size_1}).

The energy error usually decreases with the particle count, because we keep the particle density constant. Thus, when extra particles are added, many long-range interactions are created, but the number of close interactions each particle has stays roughly the same. Close interactions are more error-prone due to much stronger forces at close distances.

\subsubsection{Small Batch Size.}

\begin{figure*}[t] 
  \centering
  \includegraphics[width=1.0\linewidth]{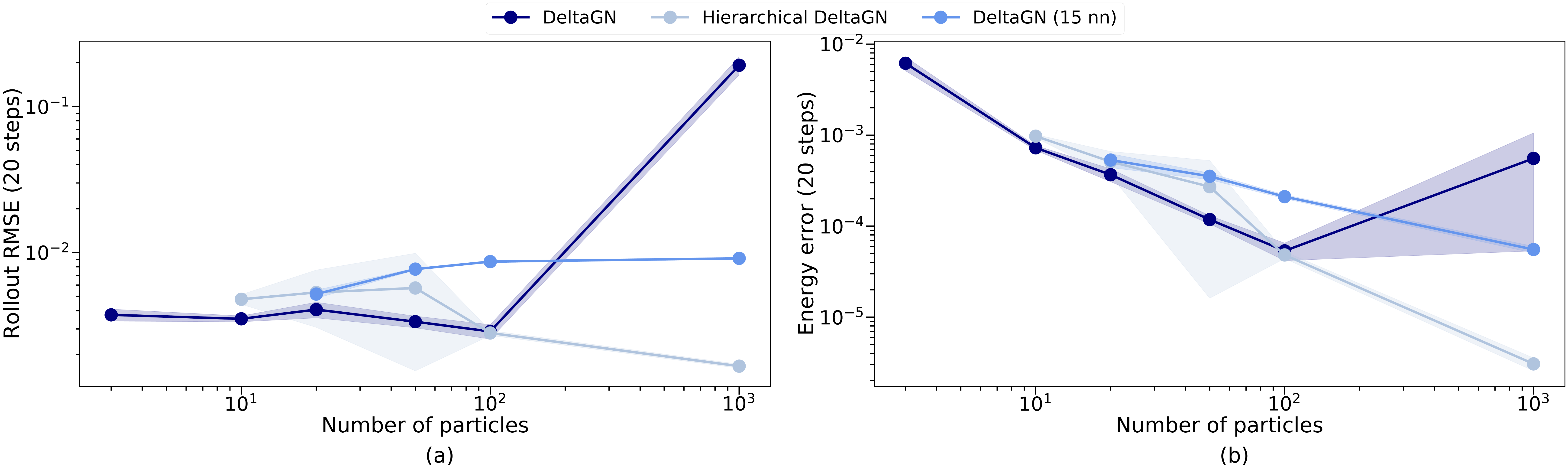}
  \caption{Models trained using a batch size of $1$. Their 20 step rollout RMSE (a) and energy error (b).}
  \label{fig:batch_size_1}
\end{figure*}

\begin{table*}[t]
\centering
\resizebox{\textwidth}{!}{\begin{tabular}{||c c c c c c||} 
 \hline
 Model & \begin{tabular}{c}
   Number of \\
   particles
 \end{tabular} &
  \begin{tabular}{c}
   Rollout RMSE\\
   (20 steps)
 \end{tabular}&
  \begin{tabular}{c}
   Energy error \\
   (20 steps)
 \end{tabular}
 & \begin{tabular}{c}
   Rollout RMSE\\
   (200 steps)
 \end{tabular}& \begin{tabular}{c}
   Energy error \\
   (200 steps)
 \end{tabular}  \\ [0.5ex] 
 \hline\hline
 \begin{tabular}{c}
   DeltaGN 
 \end{tabular} & 100 & \boldmath{$0.0045 \pm 2.7 \cdot 10^{-4}$} & \boldmath{$0.0086 \pm 0.0041$} & \boldmath{$0.0684 \pm 0.0031$} & \boldmath{$0.0584 \pm 0.0202$} \\ \hline
 \begin{tabular}{c}
   Hierarchical \\
   DeltaGN 
 \end{tabular} & 100 & $0.0052 \pm 1.61 \cdot 10^{-4}$ & $0.0152 \pm 0.0036$ & $0.1013 \pm 0.0019$ & $0.0931 \pm 0.0138$  \\ \hline
 \hline
 \begin{tabular}{c}
   DeltaGN 
 \end{tabular} & 1000 & $0.1744 \pm 0.0234$ & $ 0.0592 \pm 0.0507$ & NA & NA \\ \hline
 \begin{tabular}{c}
   Hierarchical \\
   DeltaGN 
 \end{tabular} & 1000 & \boldmath{$0.0058 \pm 1.45 \cdot 10^{-4}$} & \boldmath{$ 0.0014 \pm 1.90 \cdot 10^{-4}$} & \boldmath{$0.1252 \pm 0.0020$} & \boldmath{$0.0418 \pm 0.0031$} \\ [1ex] 
 \hline
\end{tabular}}
\caption{Test accuracy achieved by DeltaGN models on the Coulomb force datasets. DeltaGN used maximum possible batch size of 4 on the 1000 particle dataset, while in all other cases batch size of 100 was used. NA means that during unroll for some of the runs particle velocities grew so large that float32 overflowed and model returned NaN values.}
\label{table:coulomb}
\end{table*}

One easy way to reduce memory usage and to decrease the computation time is to reduce the batch size. However, this potentially comes with a big accuracy penalty. As can be seen from the rollout RMSE (Figure \ref{fig:batch_size_1}a) and the energy error (Figure \ref{fig:batch_size_1}b) plots training models with a batch size of $1$ results in a $\sim 10\%$ decrease in accuracy for DeltaGN (15 nn) and a $\sim 5$ time decrease in accuracy for the other models.

We note that using the same batch size is not entirely fair to models with $O(N)$ complexity. Indeed, they have much fewer edge samples in the batch, compared to the model that uses a fully-connected graph. This lack of edge samples is most likely the reason why some runs of the hierarchical DeltaGN on $20$ and $50$ particle datasets got stuck in local minima and caused high variance.

\subsubsection{Coulomb Interactions.}

We expect that Coulomb interactions are harder to learn since they can be either attractive or repulsive. We also made this dataset more complex by simulating charges of different magnitudes.
In Table \ref{table:coulomb} we can see that on the $100$ particle dataset the hierarchical DeltaGN again performs similarly to the model that uses a fully-connected graph. However, when the particle count increases to $1000$ DeltaGN fails, while the hierarchical DeltaGN retains similar performance. 

The particle charge was supplied to the models alongside the features supplied in the case of gravitational force.

\subsubsection{Hierarchical HOGN.}

Five independent runs of hierarchical HOGN were trained on the $100$ and $1000$ particle datasets. As can be seen from Table \ref{table:HOGN} these runs had very high variance. However, one of the hierarchical HOGN runs resulted in the most accurate model we have trained on the $1000$ particle dataset. Although, the Hamiltonian and ODE biases did not bring as large of an improvement, as seen when a fully connected graph is used \cite{sanchezgonzalez2019hamiltonian}.

We do believe that the training process of the hierarchical HOGN can be made more stable, but we leave this direction for future work.

\begin{table*}[t]
\centering
\resizebox{\textwidth}{!}{\begin{tabular}{||c c c c c c c||} 
 \hline
 Model & \begin{tabular}{c}
   Number of \\
   particles
 \end{tabular} & \begin{tabular}{c}
   Batch \\
   size
 \end{tabular} &
  \begin{tabular}{c}
   Rollout RMSE\\
   (20 steps)
 \end{tabular}&
  \begin{tabular}{c}
   Energy error \\
   (20 steps)
 \end{tabular}
 & \begin{tabular}{c}
   Rollout RMSE\\
   (200 steps)
 \end{tabular}& \begin{tabular}{c}
   Energy error \\
   (200 steps)
 \end{tabular}  \\ [0.5ex] 
 \hline\hline
 \begin{tabular}{c}
   Hierarchical \\
   DeltaGN  
 \end{tabular} & 100 & 100 & $0.0016 \pm 1.74 \cdot 10^{-5}$ & $(2.71 \pm 0.18)\cdot 10^{-5}$ & $0.0612 \pm 0.0015$ & $(2.82 \pm 0.63)\cdot 10^{-4}$ \\ \hline
 \begin{tabular}{c}
   HOGN 
 \end{tabular} & 100 & 50 & \boldmath{$(2.09 \pm 0.26) \cdot 10^{-4}$} & \boldmath{$(6.30 \pm 2.14) \cdot 10^{-6}$} & \boldmath{$0.0139 \pm 0.0023$} & \boldmath{$(4.41 \pm 1.03) \cdot 10^{-5}$} \\ \hline
 \begin{tabular}{c}
   Hierarchical \\
   HOGN 
 \end{tabular} & 100 & 100 & $0.0046 \pm 0.0042$ & $(8.57 \pm 8.48) \cdot 10^{-5}$ & $0.1145 \pm 0.0814$ & $0.0016 \pm 0.0015$  \\ \hline
 \begin{tabular}{c}
   Hierarchical \\
   HOGN (best)
 \end{tabular} & 100 & 100 & $0.0017$ & $3.06 \cdot 10^{-5}$ & $0.0579$ & $3.55 \cdot 10^{-4}$  \\ \hline
 \hline
 \begin{tabular}{c}
   Hierarchical \\
   DeltaGN 
 \end{tabular} & 1000 & 100 & $0.0013 \pm 1.99 \cdot 10^{-5}$ & $(2.37 \pm  0.40) \cdot 10^{-6}$ & $0.0482 \pm 0.0011$ & $(4.22 \pm 1.44) \cdot 10^{-5}$ \\
 \hline
 \begin{tabular}{c}
   Hierarchical \\
   HOGN 
 \end{tabular} & 1000 & 20 & $0.0023 \pm 0.0011$ & $(4.87 \pm 1.99) \cdot 10^{-6}$ & $0.0691 \pm 0.0299$ & $(4.27 \pm 1.42) \cdot 10^{-5}$ \\
 \hline
 \begin{tabular}{c}
   Hierarchical \\
   HOGN (best)
 \end{tabular} & 1000 & 20 &\boldmath{$0.0012$} & \boldmath{$2.03 \cdot 10^{-6}$} & \boldmath{$0.0377$} & \boldmath{$2.87 \cdot 10^{-5}$} \\ [1ex] 
 \hline
\end{tabular}}
\caption{Hierarchical HOGN. Due to the very high variance of the hierarchical HOGN runs we also provide the results for the best run. All models were trained with a batch size of $100$ when possible and largest batch size otherwise. HOGN ran out of memory on the $1000$ particle dataset even with a batch size of $1$.}
\label{table:HOGN}
\end{table*}

\subsubsection{Empirical Time Complexity.}
We timed 10 000 runs of the model forward pass, using different graphs as inputs (Figure \ref{fig:time_complexity}a). We observe, that while GPU parallelism initially counteracts the increased complexity, as expected using a fully-connected graph results in asymptotically quadratic time complexity. While using hierarchical and 15 nearest neighbour graphs results in asymptotically linear scaling. The pre-computation time of the nearest neighbour graphs is asymptotically quadratic, while the pre-computation time of our hierarchical graphs scales as $O(N\log N)$ (Figure \ref{fig:time_complexity}b). In both cases, pre-computation on the CPU was faster than on the GPU with our implementation. Time scaling is not monotonic for the hierarchical graph pre-computation and the hierarchical DeltaGN, because the depth of the hierarchy is set as $\lfloor\log_4(N)\rceil$.

\begin{figure*}[t] 
  \centering
  \includegraphics[width=1.0\linewidth]{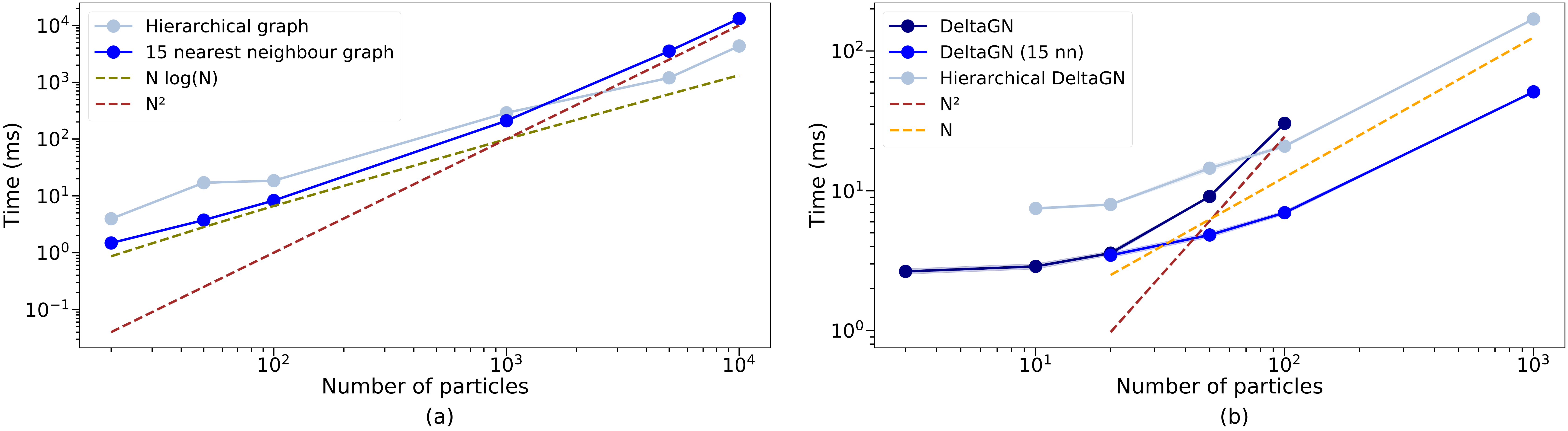}
  \caption{(a) Time scaling of the model forward pass using a batch size of $100$. (b) Time scaling of graph pre-computation on a CPU. Averaged over 10 000 (a) and 1000 (b) runs respectively.}
  \label{fig:time_complexity}
\end{figure*}

\subsubsection{Generalisation to Unseen Particle Counts.}
Models that use a fully-connected graph suffer from disastrous performance when they are evaluated on particle counts they were not trained on (Figure \ref{fig:transferability_20_steps}). The most likely cause is the sharp change in the number of incoming edges for each node. This is corroborated by the fact that DeltaGN (15 nn), which uses a graph with a constant number of neighbours, achieves almost the same accuracy on unseen particle counts as the models trained on that particle count. In our hierarchy, we have a roughly constant number of edges per node (Theorem \ref{theorem:edges}). As a result, the hierarchical DeltaGN generalises to unseen particle counts much better than models that use a fully-connected graph.  When more than $100$ particles are used, it outperforms DeltaGN (15 nn). Our hierarchical approach can be combined with existing techniques, such as randomly dropping edges during training \cite{rong2019dropedge} or constraining message size \cite{cranmer2019learning}, to further improve the generalisation.

\begin{figure}[ht!]
    \centering
    \includegraphics[width=1.0\columnwidth]{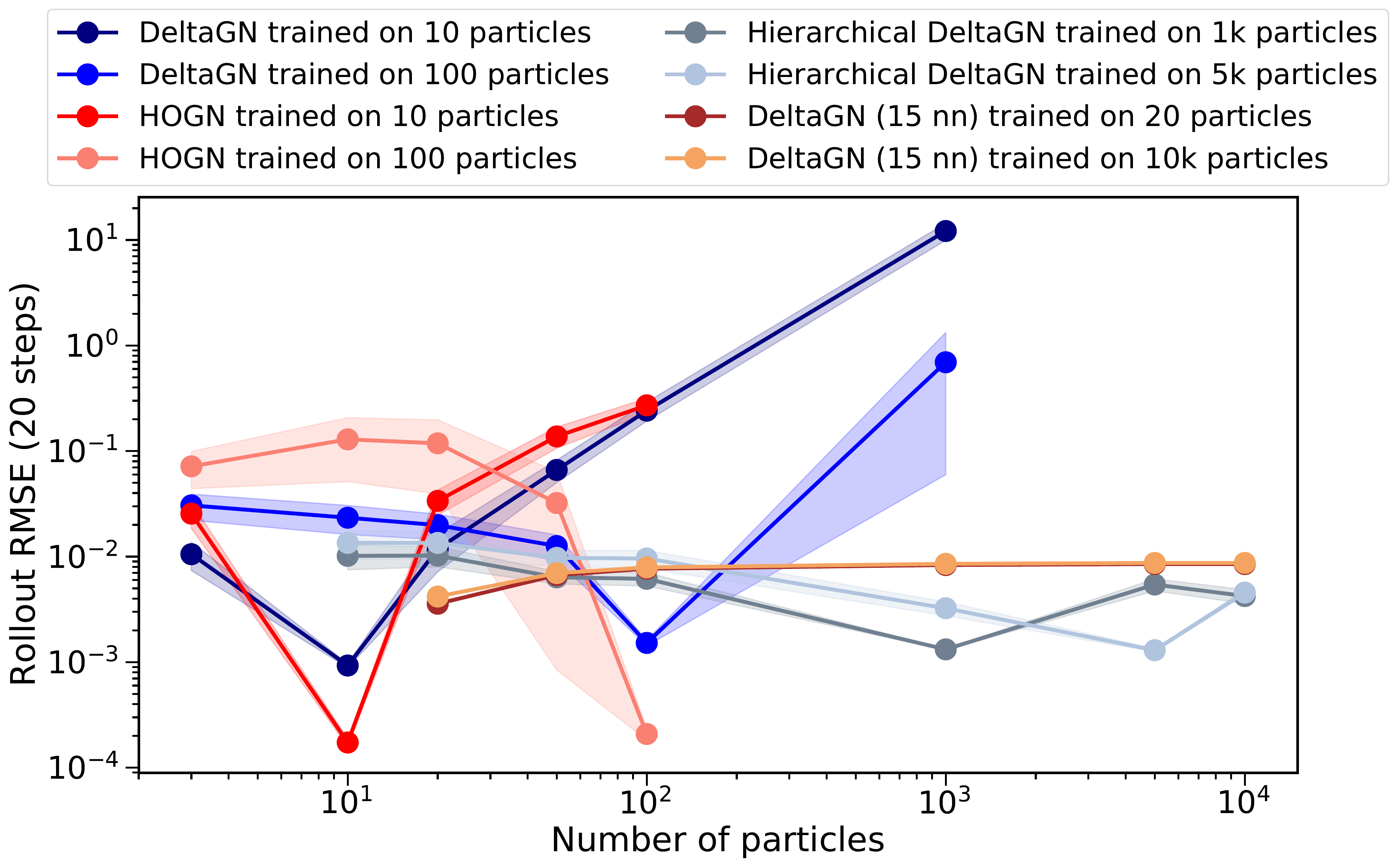}
    \caption{Models trained on one particle count evaluated on different particle counts. 20 step rollout RMSE. The trained models from Figure \ref{fig:batch_size_100} were used.}
    \label{fig:transferability_20_steps}
\end{figure}

\section{Conclusion}
We presented a novel hierarchical graph construction technique that can be used to adapt existing graph network models. We show that this hierarchical graph improves the model time and space complexity from $O(N^2)$ to $O(N)$ when applied to particle simulations that require $O(N^2)$ interactions. A pre-computation step also requires $O(N \log N)$ time and $O(N)$ space. This theoretical improvement is also observed in practice and allows us to train models on much larger datasets than previously possible. We were able to achieve good accuracy on a dataset with 10 000 particles, while prior models that use a fully-connected graph failed on 1000 particles. Our approach also displayed much better accuracy than a standard baseline based on the nearest neighbour graph. Furthermore, we observed improved generalisation to different particle counts when the hierarchy is used. We hypothesise this is due to the regularisation enforced by the hierarchy (the number of edges per particle tends to be more constant), although this requires more investigation.

While we mostly focused on the faster model that directly predicts the system's state change, we have also shown that this hierarchical approach is compatible with Hamiltonian and ODE biases. However, the inclusion of these biases does not seem to have the desired impact on the model's accuracy and further work is required to reduce the variance of the resulting model.

We tested our approach in a noiseless setting. However, noisy observations are a tangent problem that can be addressed by unrolling the trajectories for multiple steps during training and by using a more robust integration scheme \cite{desai2020vign}.

Finally, this work takes an important step towards learning to simulate larger and more realistic dynamical systems in a way that is compatible with many existing state-of-the-art approaches. Developing neural simulators is important because it can even lead to the discovery of novel physics formulas \cite{cranmer2020discovering}. In general, this method could be used to extend graph networks that are applied to other problems where global information is needed.

\section*{Acknowledgements}
The authors would like to thank Roberta Huang, Thomas Hofmann, Janis Fluri, Tomasz Kacprzak and Alexandre Refregier for their involvement in the early phase of this project.

\bibliography{refs.bib}


\appendix

\section{Data Generation}
\label{appendix:data}
As mentioned in the Data section, we built our own N-body simulator. It is based on the best practices of N-body simulations as outlined by \cite{dehnen2011n}. We refer the reader to that work for more information on the concepts mentioned below.

We model a 2D system ($\bm{q}_i = [x_i,y_i]$, $\dot{\bm{q}}_i = [\dot{x}_{i}, \dot{y}_{i}]$) and store each particle state as a 5-element vector $[m, x, y, \dot{x}, \dot{y}]$. The initialisation is described in the main body of the paper. We use periodic boundary conditions, which means that we simulate a unit cell and then tile an infinite system with copies of that cell (Figure \ref{fig:pbc}). The size of the unit cell is chosen such that we would have roughly one particle per twelve square units.

\begin{figure}[ht]
    \centering
    \includegraphics[width=1.0\columnwidth]{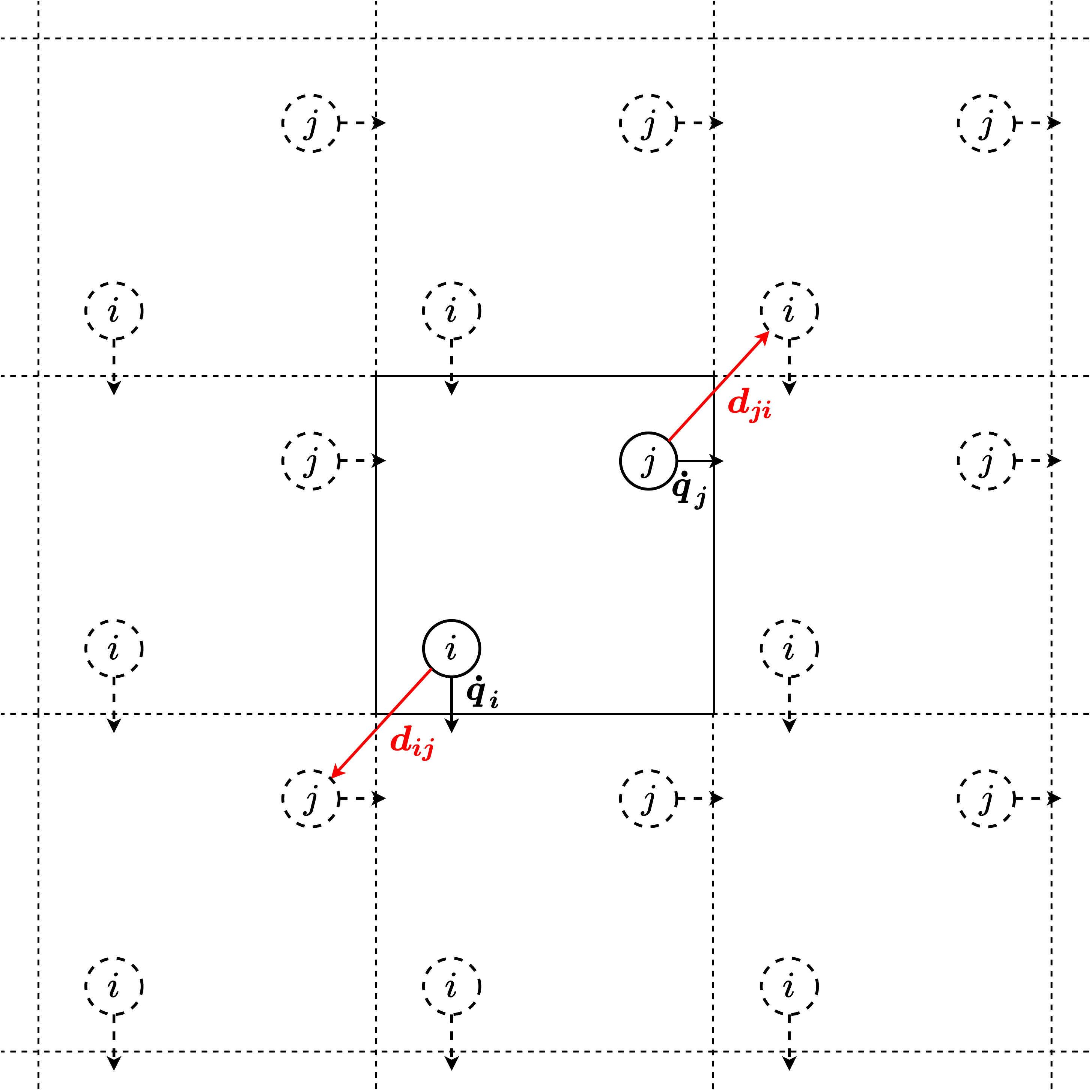}
    \caption{Periodic boundary conditions. Unit cell (solid) is tiled to create an infinite system (dashed). Distance (red) is computed to the closest copy of the particle.}
    \label{fig:pbc}
\end{figure}

The equations of motion are integrated using a symplectic Leapfrog integrator which ensures energy conservation throughout the simulation.
We use the \textit{kick drift kick} formulation of the Leapfrog integrator:
\begin{linenomath*}
\begin{equation*}
    \dot{\bm{q}}' = \dot{\bm{q}}_0 + \frac{1}{2}\bm{a}_0\Delta t \; ,
\end{equation*}
\end{linenomath*}
\begin{linenomath*}
\begin{equation*}
    \bm{q}_1 = \bm{q}_0 + \dot{\bm{q}}'\Delta t  \; ,
\end{equation*}
\end{linenomath*}
\begin{linenomath*}
\begin{equation*}
    \dot{\bm{q}}_1 = \dot{\bm{q}}' + \frac{1}{2}\bm{a}_1\Delta t  \; ,
\end{equation*}
\end{linenomath*}
where $\bm{a}$ is the acceleration vector, $\bm{q}_0$ and $\dot{\bm{q}_0}$ are the initial particle position and velocity vectors, $\Delta t$ is the integration time step.
To save particle position and velocity at the same time we use the time synchronised version of this integrator:
\begin{linenomath*}
\begin{equation*}
    \bm{q}_1 = \bm{q}_0 + \dot{\bm{q}}_0\Delta t + \frac{1}{2}\bm{a}_0\Delta t  \; ,
\end{equation*}
\end{linenomath*}
\begin{linenomath*}
\begin{equation*}
    \dot{\bm{q}}_1 = \dot{\bm{q}}_0 + \frac{1}{2}(\bm{a}_0 + \bm{a}_1)\Delta t  \; .
\end{equation*}
\end{linenomath*}

The particle acceleration is computed precisely by using all $N-1$ interactions of the particle.

To optimise the simulation we use hierarchical time-steps:
\begin{linenomath*}
\begin{equation*}
    \Delta t_n = \frac{\Delta t_0}{2^n}  \; ,
\end{equation*}
\end{linenomath*}
where $n$ is the level in the time-step hierarchy and $\Delta t_0$ is the base time step.

To assign a time-step to a particle $i$ we use the following criterion:
\begin{linenomath*}
\begin{equation}
    \Delta t_i = \eta \sqrt{\frac{\epsilon}{|\bm{a}_i|}}  \; ,
\end{equation}
\end{linenomath*}
where $\epsilon$ is the force softening length that we will discuss in the next paragraph, $|\bm{a}_i|$ is the magnitude of the particle acceleration and $\eta$ is free parameter that we set to $0.001$. The particle is assigned to the largest time-step level $\Delta t_n$ that is smaller than $\Delta t_i$. We use a criterion based solely on acceleration because higher-order estimates are not available in the Leapfrog integration scheme.
The particle state is saved only every base time-step $\Delta t_0$.

When particles are very close, their mutual gravitational attraction goes to infinity. Which means that when particles are very close we would need infinitesimal time steps to accurately integrate their trajectories. To make the simulation smoother and to avoid the singularity we use Plummer softening. In this case each particle is replaced by a Plummer sphere \cite{plummer1911problem} of scale (softening) radius $\epsilon$ and mass $m_i$. The density of the sphere at the distance $r$ from its centre is
\begin{linenomath*}
\begin{equation*}
    \rho_p(r) = \frac{3m_i}{4 \pi \epsilon^3} \left( 1 + \frac{r^2}{\epsilon^2} \right) \; .
\end{equation*}
\end{linenomath*}

This results in a softened acceleration 
\begin{linenomath*}
\begin{equation*}
    \bm{a}_i = - G \sum_{j \neq i} \frac{m_j}{(\|\bm{q}_i - \bm{q}_j\|_2^2 + \epsilon^2)^{\frac{3}{2}}} \cdot (\bm{q}_i - \bm{q}_j)  \; ,
\end{equation*}
\end{linenomath*}
that for a pair of particles goes to zero when distance is smaller than $\epsilon$. We found $\epsilon = 0.2$ to be sufficiently large to avoid the need of very small time steps, while being sufficiently small to not influence most particle interactions.

If Coulomb interactions are simulated, the acceleration is instead computed as 
\begin{linenomath*}
\begin{equation*}
    \bm{a}_i = k \frac{1}{m_i} \sum_{j \neq i} \frac{c_i \cdot c_j}{(\|\bm{q}_i - \bm{q}_j\|_2^2 + \epsilon^2)^{\frac{3}{2}}} \cdot (\bm{q}_i - \bm{q}_j)  \; ,
\end{equation*}
\end{linenomath*}
where $c$ is the particle charge, and $k$ is the Coulomb constant. The particle charge is also appended to the particle state vector.
 
\section{Model Implementation}
\label{appendix:models}
Graph network block \cite{battaglia2018relational} architectures used in DeltaGN and HOGN models \cite{sanchezgonzalez2019hamiltonian} can be seen in Figure \ref{fig:DeltaGN_block} and Figure \ref{fig:HOGN_block} respectively. It's assumed that we are working with a graph $G=(V,E,\bm{u})$, with global features $\bm{u}$, nodes $V$ and edges $E$. $N^v$ is the number of nodes and $N^e$ is the number of edges. Each edge has its sender $s_k$ and receiver $r_k$ IDs. All aggregation functions $\rho$ are summations. $\rho^{e \rightarrow v}$ aggregates all of the incoming edges for each node, $\rho^{e \rightarrow u}$ and $\rho^{v \rightarrow u}$ aggregate all of the edges and nodes respectively. The edge block's MLP $\phi^e$ has 2 hidden layers, each with 150 hidden units. Each layer is followed by an activation function (including the last one). The node block's MLP $\phi^v$ has 3 hidden layers, each with 100 hidden units. Each layer is again followed by an activation function. The global block's MLP $\phi^u$ which is only used in HOGN has 2 hidden layers, each with 100 hidden units. Each layer is again followed by an activation function.


The activation function used in DeltaGN is ReLU, while HOGN uses SoftPlus. ReLU does not work with the HOGN model as its derivatives are either 0 or 1. The same activation functions are used in the modified hierarchical models and the hierarchy MLPs.

Overall architectures used for HOGN and DeltaGN models can be seen in Figures \ref{fig:DeltaGN_architecture} and \ref{fig:HOGN_architecture} respectively. They do not change for the hierarchical model versions, besides the fact that we use a hierarchical graph. Both architectures use a linear layer $\phi^\text{decode}$ to transform graph network output into the coordinate change or the Hamiltonian of the system.

For HOGN we use RK4 integrator:
\begin{linenomath*}
\begin{align*}
(\bm{q}_{t+1},\dot{\bm{q}}_{t+1})  = (\bm{q}_{t},\dot{\bm{q}}_{t}) + \frac{1}{6}\Delta t (\bm{k}_1 + 2\bm{k}_2 + 2\bm{k}_3 + \bm{k}_4)\; ,\\
\!\begin{aligned}[t]\bm{k}_1&=f((\bm{q}_{t},\dot{\bm{q}}_{t})) \; ,
&\bm{k}_2&=f((\bm{q}_{t},\dot{\bm{q}}_{t}) + \frac{1}{2} \Delta t \bm{k}_1) \; ,\\
\bm{k}_3&=f((\bm{q}_{t},\dot{\bm{q}}_{t}) + \frac{1}{2}  \Delta t \bm{k}_2) \; ,            &\bm{k}_4&=f((\bm{q}_{t},\dot{\bm{q}}_{t}) + \Delta t \bm{k}_3) \; ,\end{aligned}\!\end{align*}
\end{linenomath*}
where $f$ is our neural network. The same graphs are used for all of the internal RK4 steps.

\begin{figure}[ht]
    \centering
    \includegraphics[width=1.0\columnwidth]{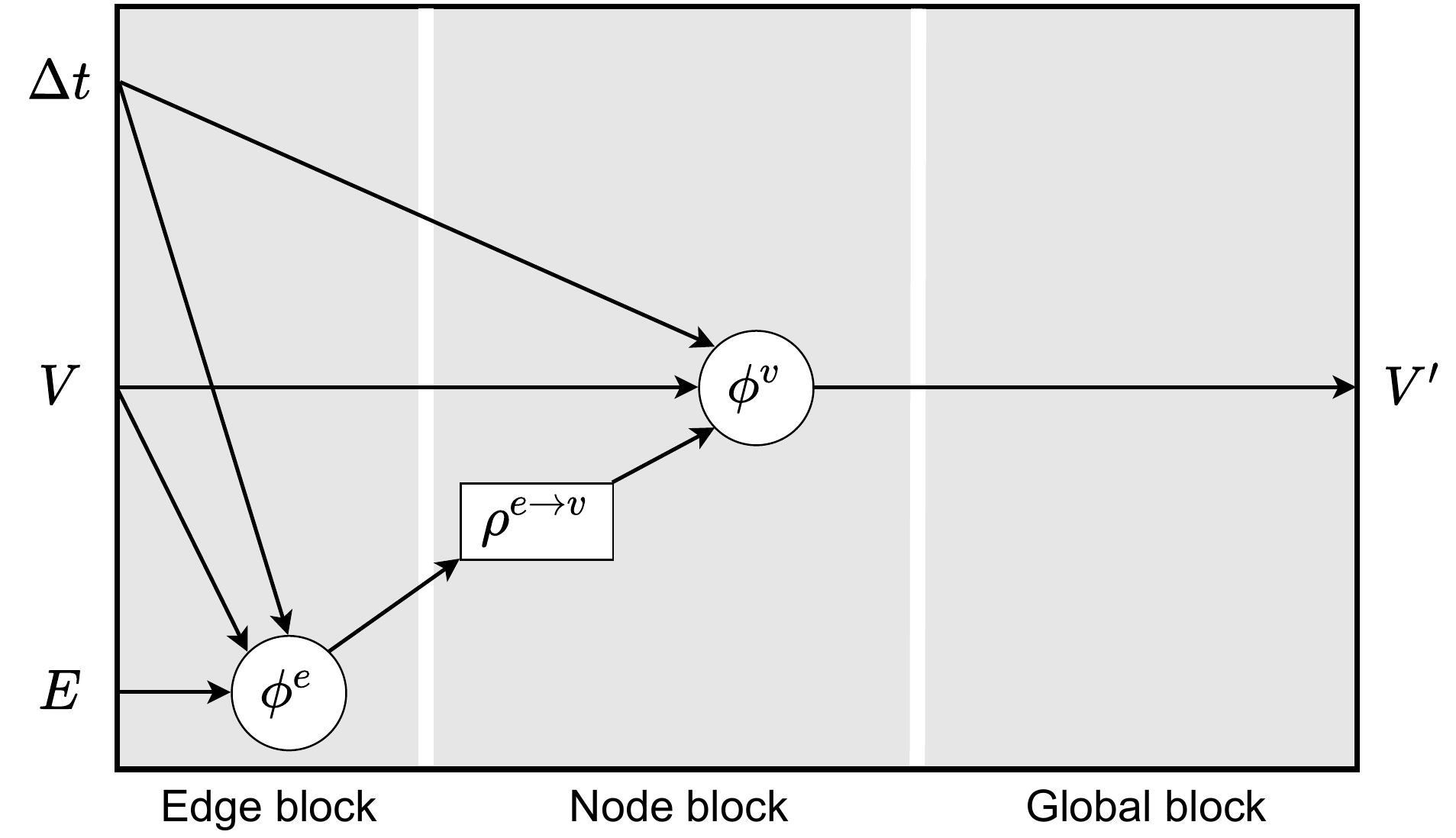}
    \caption{Graph network block scheme used in DeltaGN.}
    \label{fig:DeltaGN_block}
\end{figure}

\begin{figure}[ht]
    \centering
    \includegraphics[width=1.0\columnwidth]{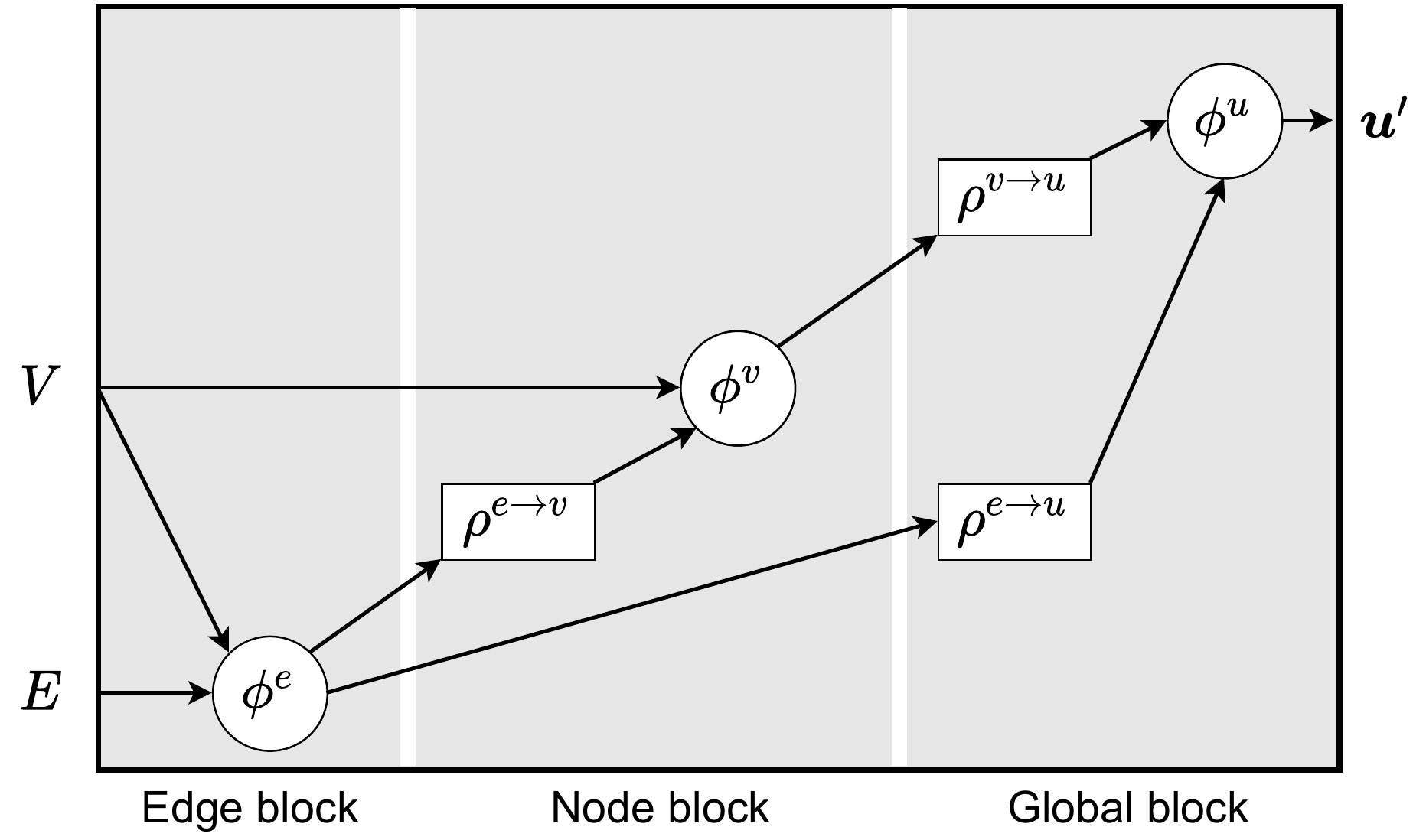}
    \caption{Graph network block scheme used in HOGN.}
    \label{fig:HOGN_block}
\end{figure}

\begin{figure}[ht]
    \centering
    \includegraphics[width=1.0\columnwidth]{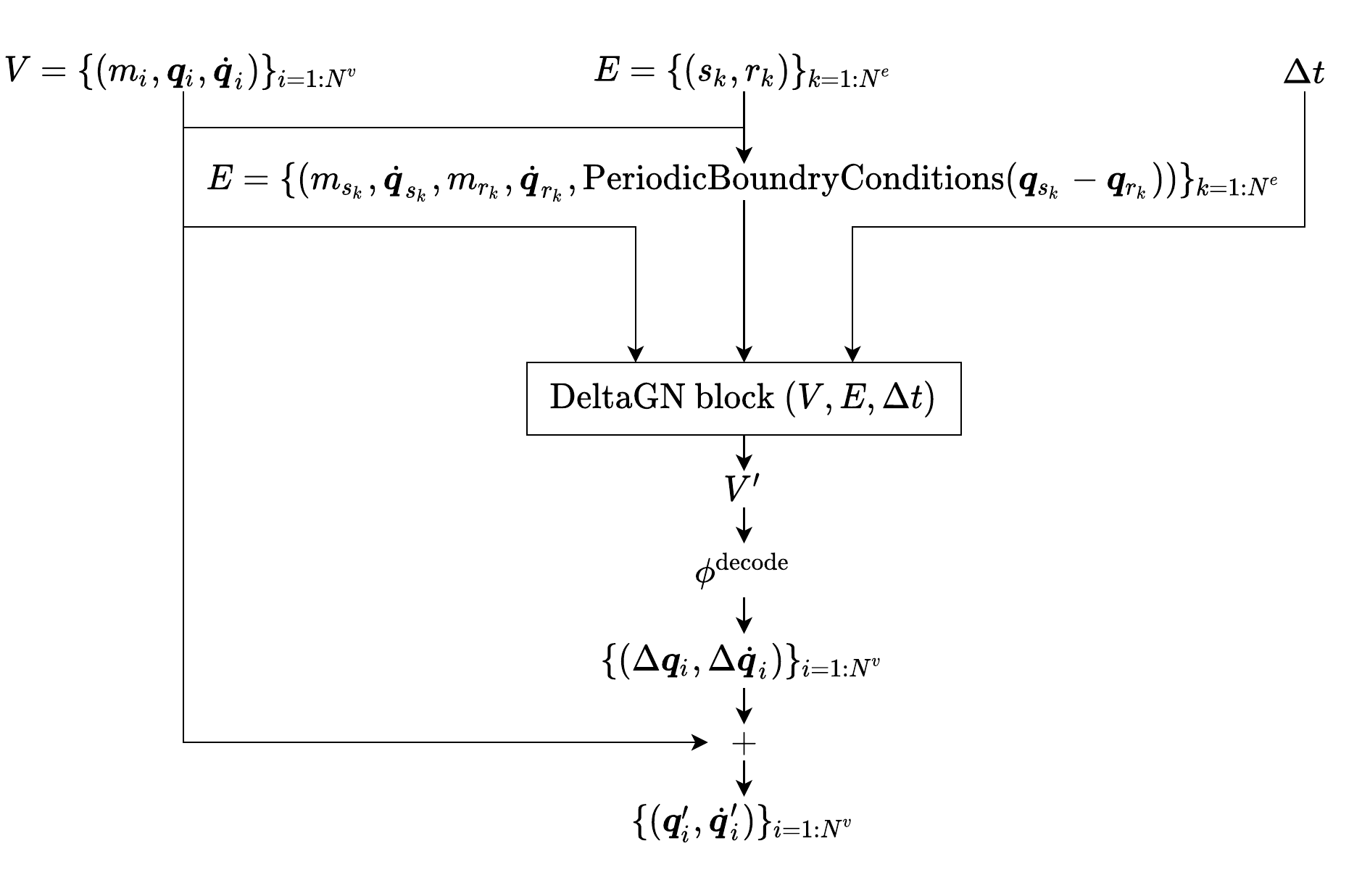}
    \caption{DeltaGN architecture.}
    \label{fig:DeltaGN_architecture}
\end{figure}

\begin{figure}[ht]
    \centering
    \includegraphics[width=1.0\columnwidth]{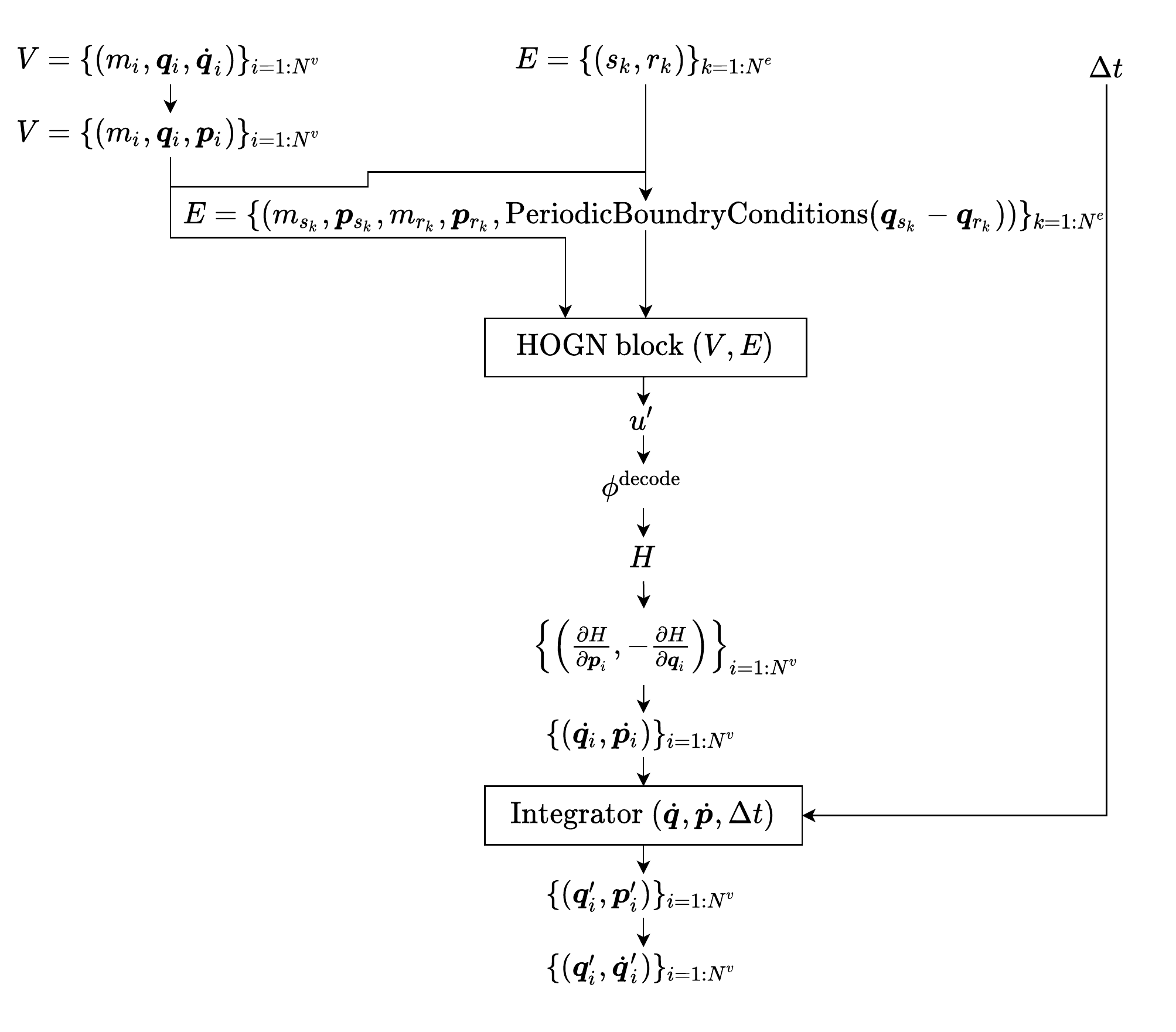}
    \caption{HOGN architecture.}
    \label{fig:HOGN_architecture}
\end{figure}

As discussed in the Adapting Existing Models section, when a hierarchical graph is used upward and downward passes must be performed first. The MLPs used in the upward pass ($\phi^{p\rightarrow c}$ and $\phi^{{c_c}\rightarrow {c_p}}$) have two hidden layers with 100 hidden units in each of them. Each layer is followed by an appropriate activation function (depending on if DeltaGN or HOGN architecture is used). MLPs used to compute interactions during the downward pass ($\phi^{c\rightarrow c}$ and $\phi^{{c_p}\rightarrow {c_c}}$) have two hidden layers with 150 hidden units in each of them. Each layer is followed by an appropriate activation function. The MLP that is used to update cell features during the downward pass ($\phi^{c}$) has three hidden layers with 100 hidden units in each of them. Each layer is again followed by an appropriate activation function. When the cell embeddings at the lowest level ($l = L-1$) are updated, we can finally construct the edge going from the parent cell to the particle, which represents all long-range interactions of the particle, using the $\phi^{{c}\rightarrow {p}}$ MLP. This MLP has two hidden layers with 150 hidden units in each of them. Each layer is followed by an appropriate activation function. How this edge is incorporated into DeltaGN and HOGN graph network blocks can be seen in Figures \ref{fig:down_block_DeltaGN} and \ref{fig:down_block_HOGN_global} respectively. The remaining block architecture is exactly the same as before.

\begin{figure}[ht]
    \centering
    \includegraphics[width=1.0\columnwidth]{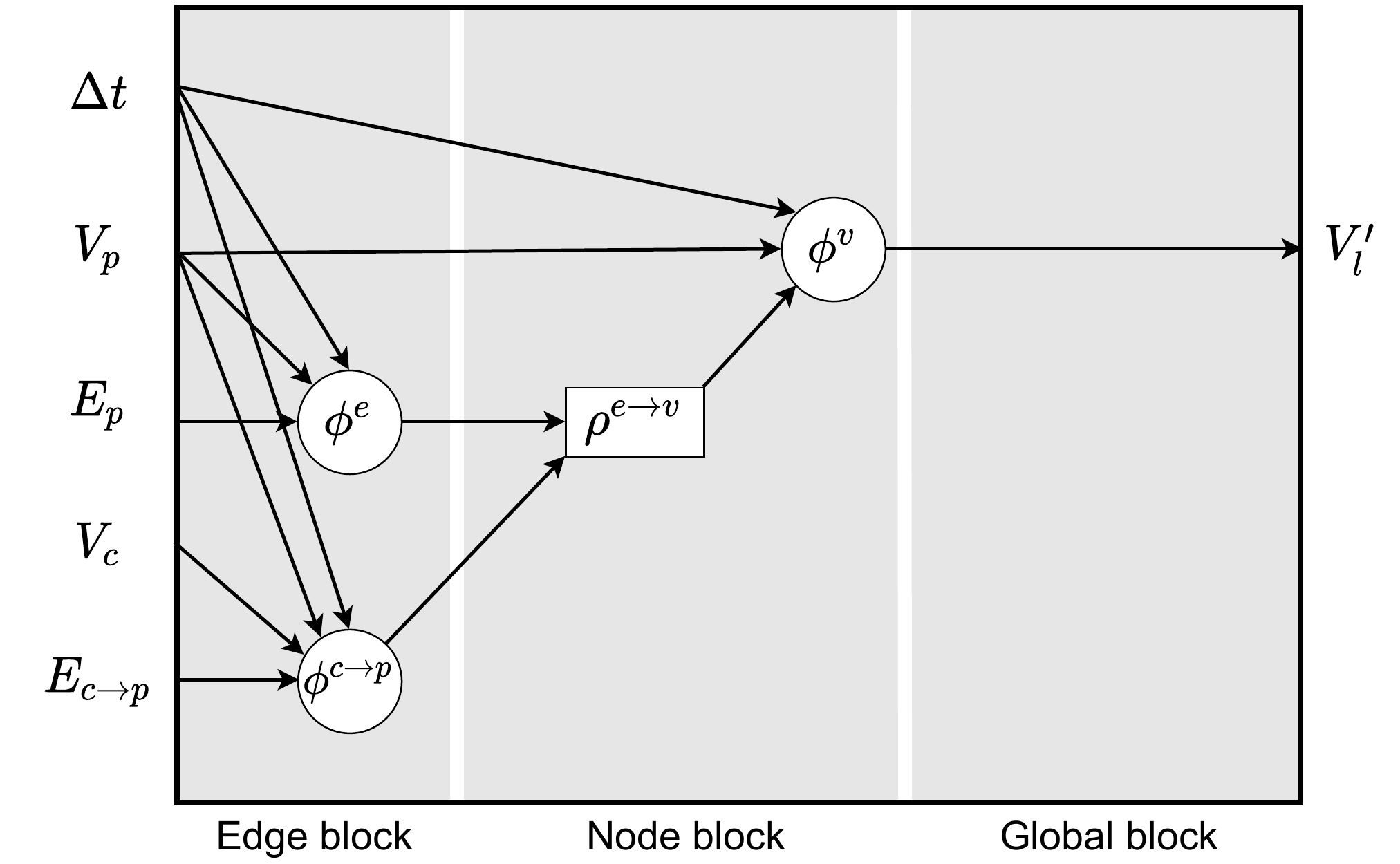}
    \caption{Graph network block scheme used in the hierarchical DeltaGN. $V_p$ and $E_p$ represent the nodes and edges of the sparse particle graph from the lowest level of the hierarchy ($l=L$). $V_c$ are the nodes of the parent cells ($l = L-1$), $E_{c \rightarrow p}$ are edges going from each cell to its child particles.}
    \label{fig:down_block_DeltaGN}
\end{figure}

\begin{figure}[ht]
    \centering
    \includegraphics[width=1.0\columnwidth]{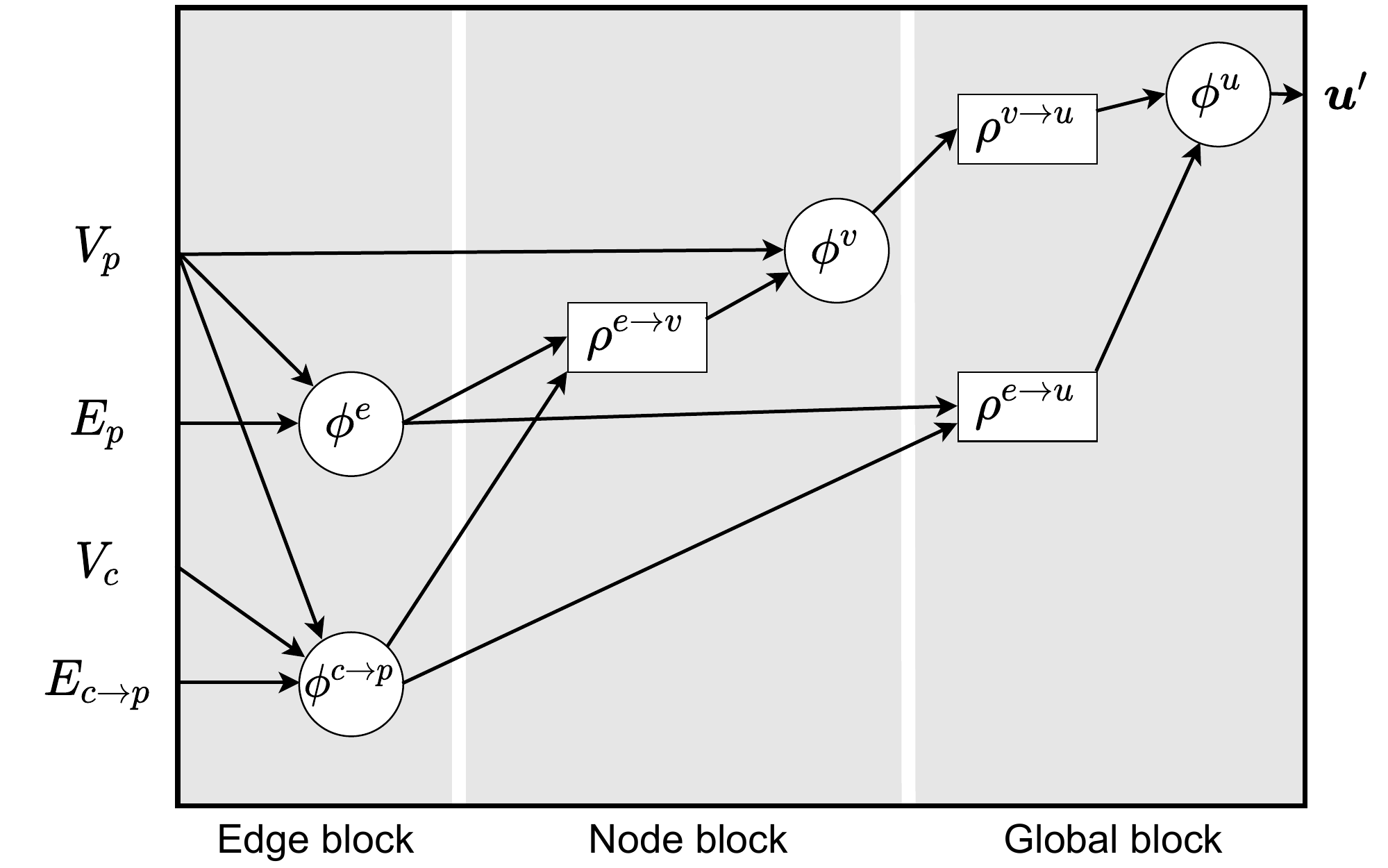}
    \caption{Graph network block scheme used in the hierarchical HOGN.  $V_p$ and $E_p$ represent the nodes and edges of the sparse particle graph from the lowest level of the hierarchy ($l=L$). $V_c$ are the nodes of the parent cells ($l = L-1$), $E_{c \rightarrow p}$ are edges going from each cell to its child particles.}
    \label{fig:down_block_HOGN_global}
\end{figure}

The training setup was described in the Results section of the main body. We did not set a fixed random seed for the models. We provide the mean error and the standard deviation when the models are run with different random seeds. As the variance is low, similar results should be observed on the datasets we will provide with any random seed.

The number of neighbours in the 15 nearest neighbour baseline was chosen such that the number of neighbours would be larger than the expected number of neighbours on the lowest level of our hierarchy (8) and would still offer good scalability.

\section{Additional Results}
\label{appendix:results}

In this section, we provide the 200 step rollout graphs, which were not presented in the Results section. We also provide visualisations of the trajectories predicted by the different DeltaGN variants. An additional experiment shows, that $log(N)$ is the optimal number of hierarchical graph levels. We also show, what accuracy can be achieved in 24 hours of training by the different models.

\subsubsection{Small Batch Size - 200 Step Rollout.}
In Figure \ref{fig:batch_size_1_200} you can see the rollout RMSE and the energy error when trajectories are unrolled for 200 steps. The plots are analogous to the error plots presented in the Results section. The error increased roughly proportionally for all of the models when trajectories were unrolled for 200 steps instead of 20.

\subsubsection{Generalisation to Unseen Particle Counts - 200 Step Rollout.}
In Figure \ref{fig:transferability_200_steps} you can see the rollout RMSE and the energy error when trajectories are unrolled for 200 steps. Compared to the shorter 20 step rollout presented in the Results section, we see that when the trajectory is unrolled for more time steps, the error of the hierarchical DeltaGN increases less than the error of the DeltaGN (15 nn).

\begin{figure}[ht]
\centering
  \begin{subfigure}[t]{1.0\linewidth}
    \centering
    \includegraphics[width=1.0\linewidth]{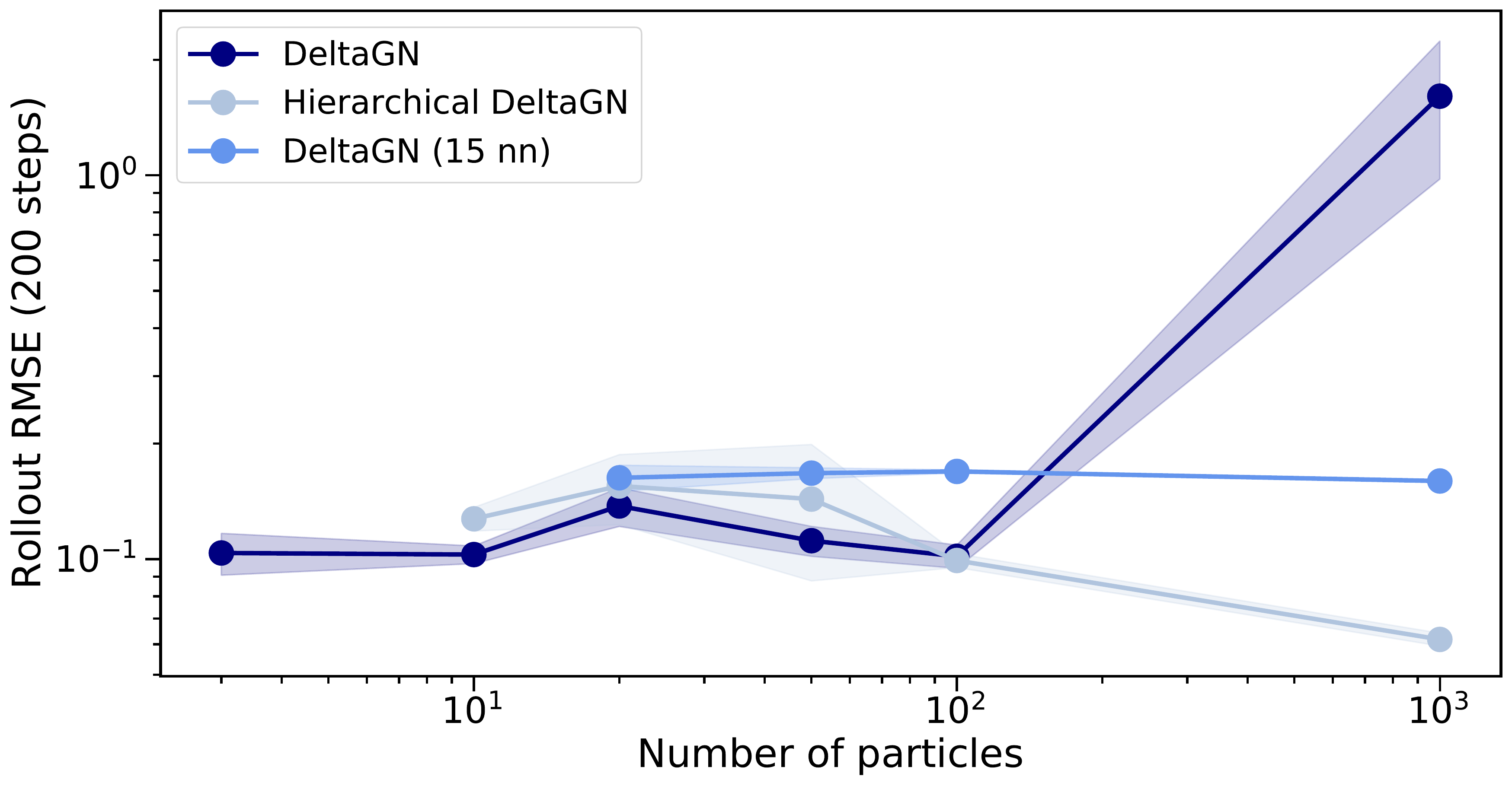}
    \vspace{-3.5ex}
    \caption{} 
    \label{batch_size_1_200:a} 
  \end{subfigure}
  \begin{subfigure}[t]{1.0\linewidth}
    \centering
    \includegraphics[width=1.0\linewidth]{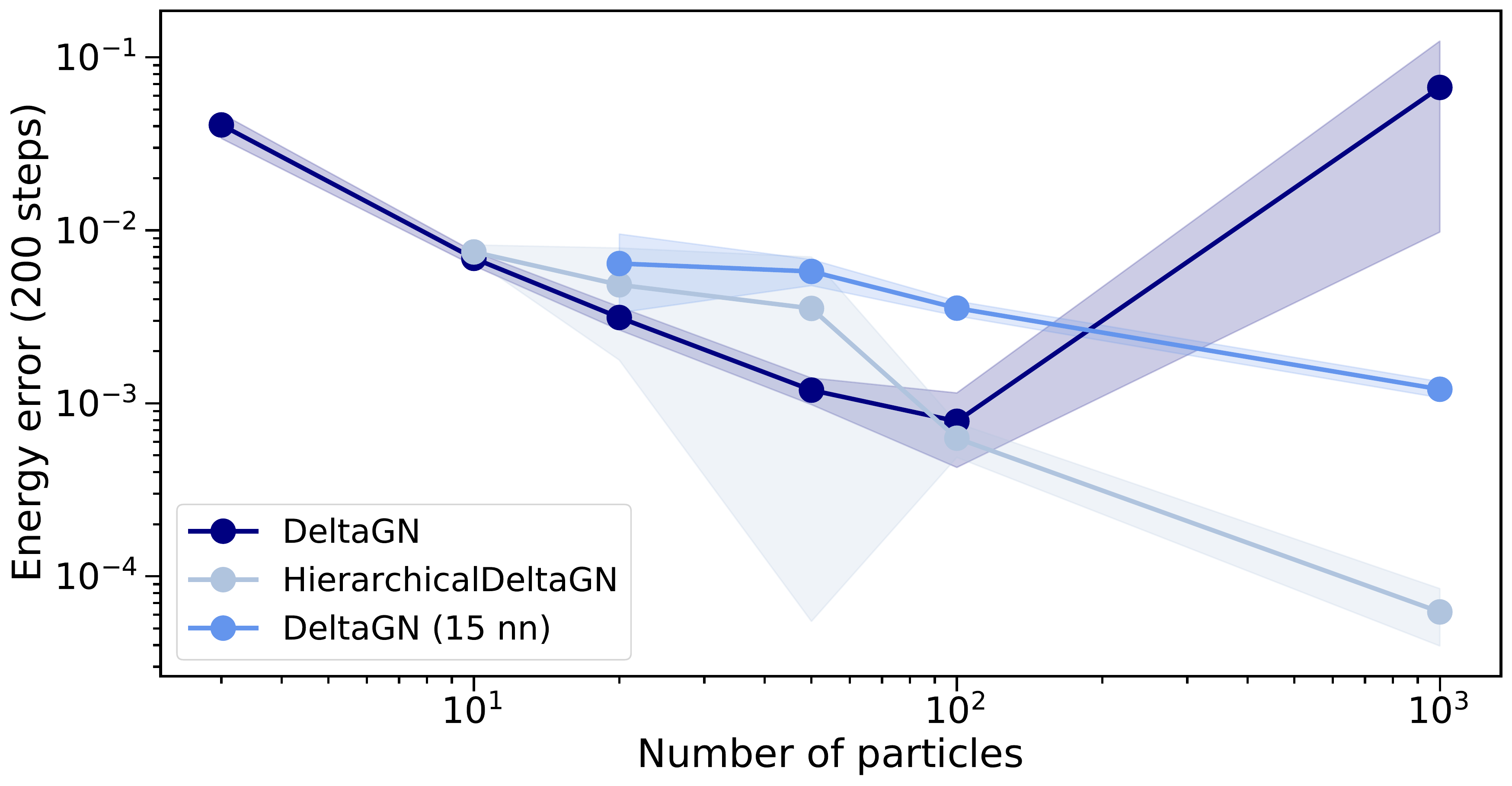}
    \vspace{-3.5ex}
    \caption{} 
    \label{batch_size_1_200:b} 
  \end{subfigure} 
  \caption{Models trained using a batch size of $1$. Their 200 step rollout RMSE (a) and energy error (b).}
  \label{fig:batch_size_1_200}
\end{figure}

\begin{figure}[ht]
    \centering
    \includegraphics[width=1.0\columnwidth]{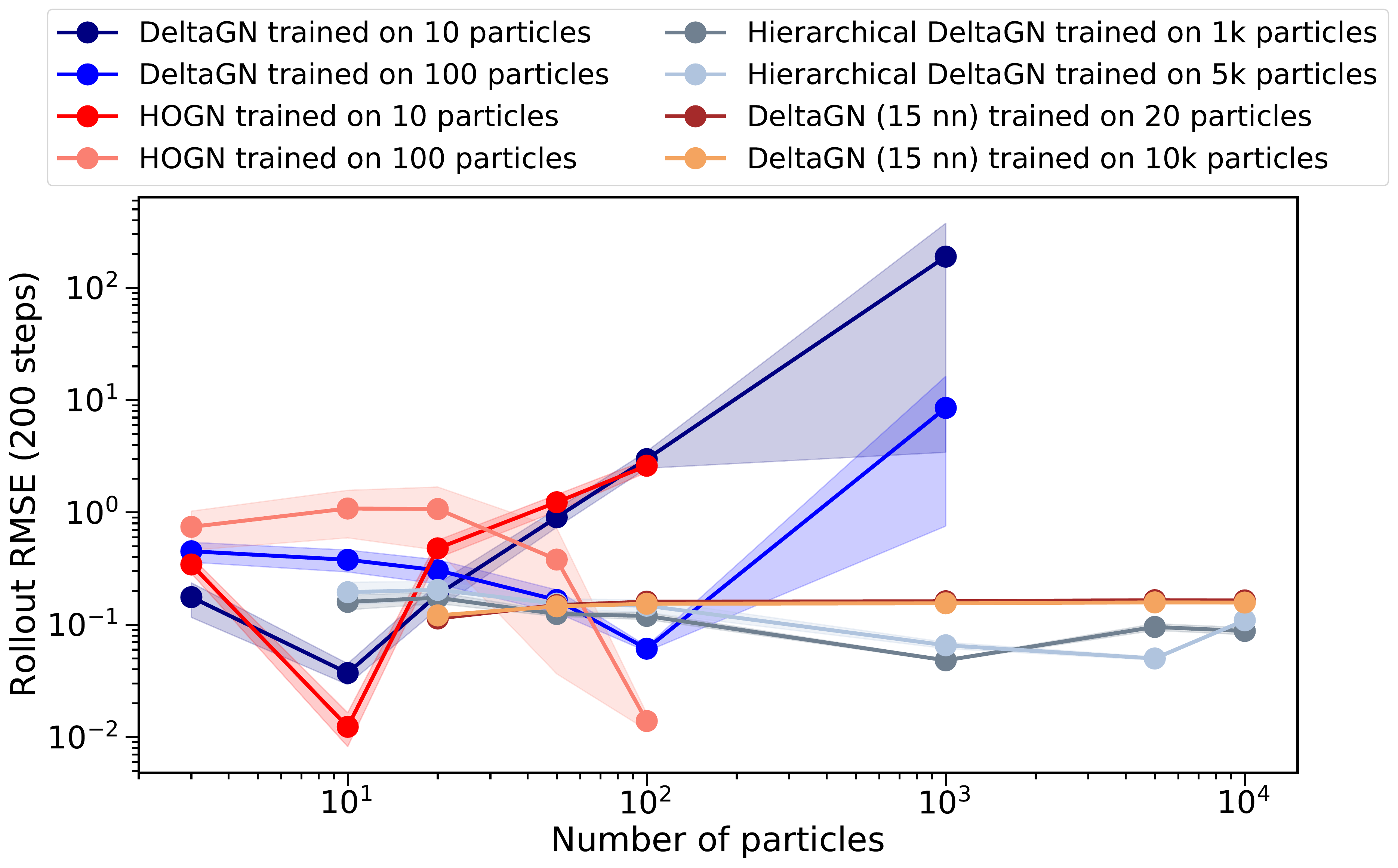}
    \caption{Models trained on a dataset with one particle count evaluated on datasets with different particle counts. 200 step rollout RMSE. HOGN was trained on 100 particle dataset using a batch size of 50. All other models were trained using a batch size of 100.}
    \label{fig:transferability_200_steps}
\end{figure}

\subsubsection{Optimal Number of Levels.}

We tested what impact the number of hierarchy levels has on accuracy and computation time. We trained multiple versions of hierarchical DeltaGN on the 100 particle dataset using 2, 3, 4 and 5 levels in the hierarchy. Note that $\log_4(100) = 3.32$ and we would normally use a 3 level hierarchy. From Table \ref{table:levels} we see that the 3 level hierarchy, in fact, provides the best combination of accuracy and speed. The model with 2 level hierarchy is slower because there are many more connections directly between the particles (in expectation 56 incoming edges per particle), while the models with 4 and 5 level hierarchies are slower because we need more steps to propagate the information through the hierarchy. These excess levels most likely are the cause of the worse accuracy which we see in case of the model that uses a 5 level hierarchy, while the poor accuracy of the model with a 2 level hierarchy might be caused by the hierarchy's MLPs not being trained as well. This is likely because the weights of the MLPs are shared for all of the hierarchy levels, except the particle level.
 
\begin{table*}[t!]
\centering
\resizebox{1.0\textwidth}{!}{\begin{tabular}{||c c c c c c c||} 
 \hline
 Model & \begin{tabular}{c}
  Hierarchy \\
  levels
 \end{tabular} &
  \begin{tabular}{c}
  RMSE $[10^{-3}]$\\
  (20 steps) 
 \end{tabular} &
  \begin{tabular}{c}
  Energy \\
  error $[10^{-5}]$\\
  (20 steps)
 \end{tabular}
 & \begin{tabular}{c}
  RMSE $[10^{-2}]$\\
  (200 steps) 
 \end{tabular} &
  \begin{tabular}{c}
  Energy\\
  error $[10^{-4}]$\\
  (200 steps)
 \end{tabular}  &
  \begin{tabular}{c}
  Forward
  pass \\ time \\
  (500k step avg.)
 \end{tabular}  \\ [0.5ex] 
 \hline\hline
 \begin{tabular}{c}
  Hierarchical \\
  DeltaGN 
 \end{tabular} & 2 levels & $1.706 \pm 0.184$ & $2.788 \pm 0.267$ & $6.612 \pm 0.476$ & $2.945 \pm 0.288$ & 28 ms  \\ \hline
 \begin{tabular}{c}
  Hierarchical \\
  DeltaGN 
 \end{tabular} & \textbf{3 levels} & $1.571 \pm 0.017$ &  $2.713 \pm 0.256$ & \boldmath{$6.118 \pm 0.147$} & \boldmath{$2.815 \pm 0.626$} &  \textbf{23 ms}  \\ \hline
 \begin{tabular}{c}
  Hierarchical \\
  DeltaGN 
 \end{tabular} & 4 levels & \boldmath{$1.545 \pm 0.023$} &  \boldmath{$2.621 \pm 0.091$} & $6.27 \pm 0.134$  & $2.910 \pm 0.236$ & 27 ms  \\ \hline
 \begin{tabular}{c}
  Hierarchical \\
  DeltaGN 
 \end{tabular} & 5 levels & $1.592 \pm 0.018$ &  $2.711 \pm 0.136$ & $6.680 \pm 0.123$  & $3.315 \pm 0.584$ & 30 ms \\[1ex] 
 \hline
\end{tabular}}
\caption{Test accuracy achieved by Hierarchical DeltaGN models with different numbers of hierarchy layers. Models were trained and evaluated on the 100 particle dataset. All models were trained using a batch size of 100. Optimal number of levels for 100 particles would be $\log_4(100) = 3.32$.}
\label{table:levels}
\end{table*}

\subsubsection{Performance Achievable in 24 hours.}

To further highlight the benefits of the hierarchical graph we train all of the versions of DeltaGN, as well as a hierarchical HOGN for 24 hours on the 1000 particle dataset. We use either batch size of 4 (which is maximum for DeltaGN) or maximum batch size possible for each model. HOGN runs out of memory even with a batch size of 1. From Table \ref{table:24hrs} we can see that DeltaGN (15 nn) allows for the biggest batch and is the fastest with the batch size of 4. However, its accuracy is much worse than that of the hierarchical models. While these hierarchical models still achieve good speed and much larger batch sizes than the DeltaGN which uses a fully connected graph.

\begin{table*}[t!]
\centering
\resizebox{0.75\textwidth}{!}{\begin{tabular}{||c c c c c||} 
 \hline
 Model & \begin{tabular}{c}
  Batch \\
  size
 \end{tabular} &
  \begin{tabular}{c}
  RMSE\\
  (20 steps)
 \end{tabular}
 & \begin{tabular}{c}
  RMSE\\
  (200 steps)
 \end{tabular} & 
 \begin{tabular}{c}
  Training \\
  steps
 \end{tabular} \\ [0.5ex] 
 \hline\hline
 DeltaGN & 4 & $0.2416 \pm 0.0362$ & NA & 239k \\  \hline
 DeltaGN (15 nn) & $4$ & $0.0086 \pm 2.15 \cdot 10^{-5}$  & $0.1565 \pm 2.35 \cdot 10^{-4}$ & \textbf{5.7M} \\  \hline
 \begin{tabular}{c}
  Hierarchical \\
  DeltaGN 
 \end{tabular} & 4 & \boldmath{$0.0014 \pm 2.73 \cdot 10^{-5}$} & \boldmath{$0.05119 \pm 0.0011$} & 1.24M \\ \hline
 \begin{tabular}{c}
  Hierarchical \\
  HOGN 
 \end{tabular} & 4 &  $0.0744 \pm 0.0948$  & $0.4156 \pm 0.3839$ & 118k  \\ \hline
 \hline
 DeltaGN & 4 & $0.2416 \pm 0.0362$  & NA & \textbf{239k} \\  \hline
 DeltaGN (15 nn) & \textbf{320} & $0.0095 \pm 4.95 \cdot 10^{-4}$  & $0.1794 \pm 0.0120$ & 151k \\  \hline
 \begin{tabular}{c}
  Hierarchical \\
  DeltaGN 
 \end{tabular} & 150 & \boldmath{$0.0043 \pm 0.0019$}  & \boldmath{$0.1176 \pm 0.0312$} & 93k \\ \hline 
 \begin{tabular}{c}
  Hierarchical \\
  HOGN 
 \end{tabular} & 20 & $ 0.0109 \pm 0.0099$ & $0.2096 \pm 0.1234$ & 50k \\[1ex] 
 \hline
\end{tabular}}
\caption{Model accuracy after 24 hours of training on the 1000 particle dataset. HOGN ran out of memory even with batch size of 1. In the upper half of the table, we have results when using batch size of 4 for all of the models. In the lower half, we have results when using maximum possible batch size for all of the models. NA means that during unroll for some of the runs particle velocities grew so large that float32 overflowed and model returned NaN values.}
\label{table:24hrs}
\end{table*}

\subsubsection{Predicted Trajectories.}

In this subsection we provide model predictions for one of the test trajectories from the 1000 particle gravitational or Coulomb dataset. Figures \ref{fig:1k_trajecotries_pred} and \ref{fig:1k_trajecotries_pred_energies} correspond to the models trained in the Scaling to Larger Particle Counts section, Figures \ref{fig:1k_trajecotries_pred_batch_1} and \ref{fig:1k_trajecotries_pred_batch_1_energies} correspond to the models trained in the Small Batch Size section, Figures \ref{fig:1k_coulomb} and \ref{fig:1k_coulomb_energies} correspond to the models trained in the Coulomb Interactions section. All of the provided predictions were made by the best corresponding model out of 5 runs.

From these trajectories, it is easy to see that DeltaGN does not learn the dynamics. While DeltaGN (15 nn) makes reasonable predictions, the lack of long-range interactions can result in large errors for some particles. For example, the particle highlighted by the red circle in Figure \ref{fig:1k_trajecotries_pred} is moving downwards instead of upwards in the trajectory predicted by DeltaGN (15 nn).

Note that the angular momentum is not conserved, because we use periodic boundary conditions \cite{Kuzkin_2014}.

\begin{figure*}[t]
    \centering
    \includegraphics[width=0.5935\textwidth]{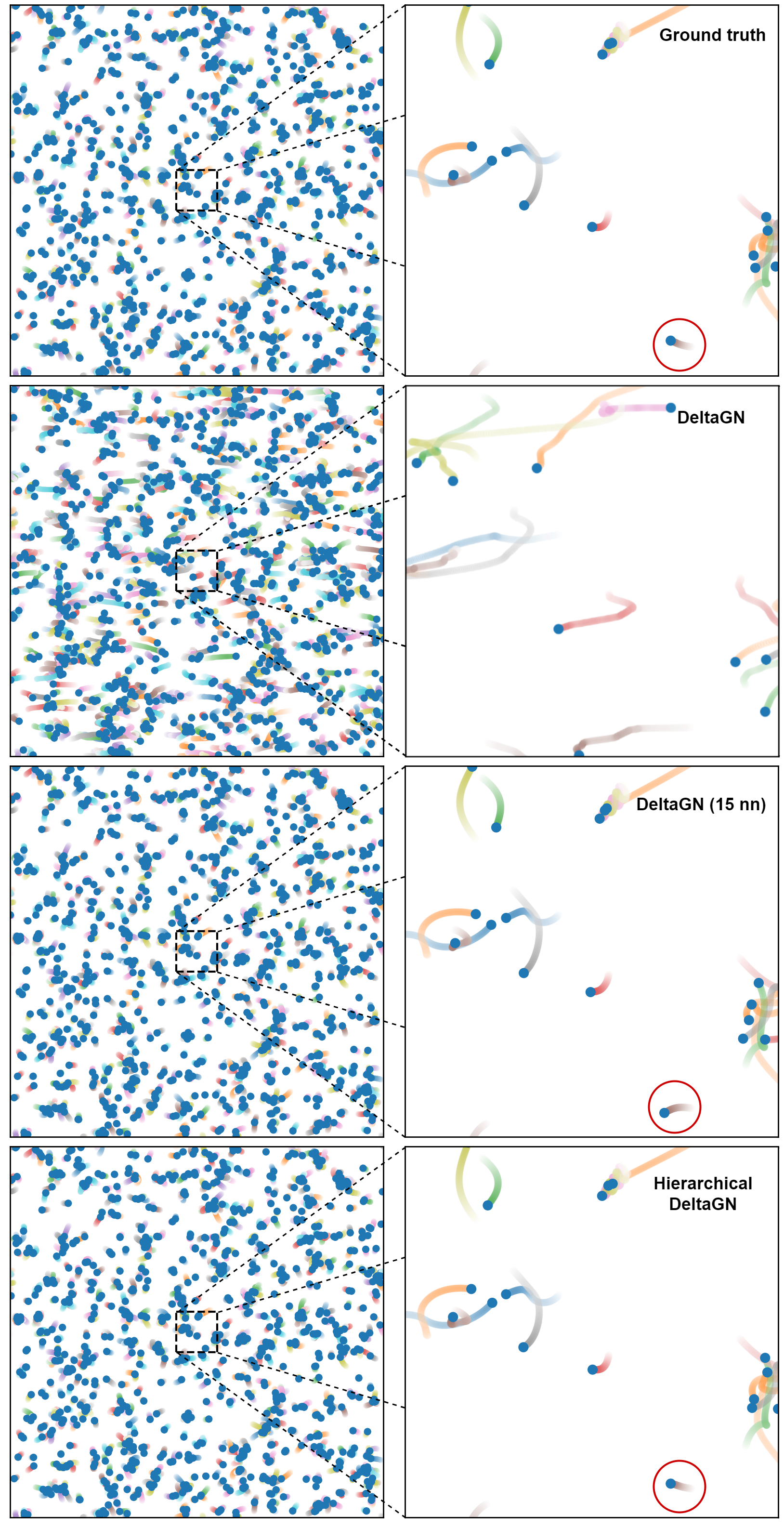}
    \caption{Sample simulation from the 1000 particle gravitational test set and the corresponding model predictions. On the right, we have a zoomed-in central region. In the first row, we show the ground truth simulation. In the second, third and fourth row we have the predictions produced by the best DeltaGN, DeltaGN (15 nn) and Hierarchical DeltaGN model runs respectively. DeltaGN was trained with a maximum possible batch size of 4, the other models used a standard batch size of 100. This plot corresponds to the models trained in the Scaling to Larger Particle Counts section.}
    \label{fig:1k_trajecotries_pred}
\end{figure*}

\begin{figure*}[t]
    \centering
    \includegraphics[width=0.575\textwidth]{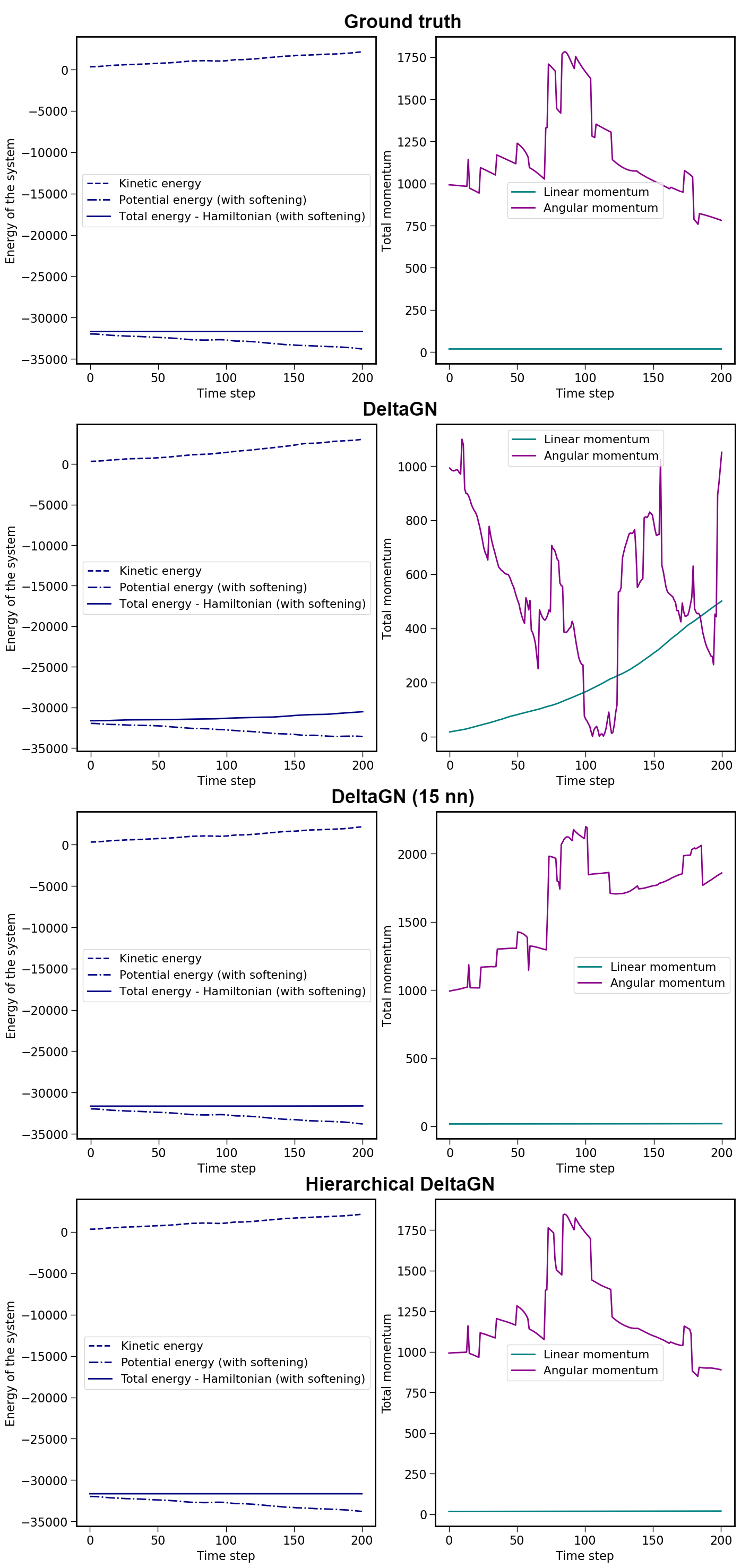}
    \caption{Energy (left column) and momentum (right column) plots corresponding to the sample simulation and predictions from Figure \ref{fig:1k_trajecotries_pred}.}
    \label{fig:1k_trajecotries_pred_energies}
\end{figure*}

\begin{figure*}[t]
    \centering
    \includegraphics[width=0.6\textwidth]{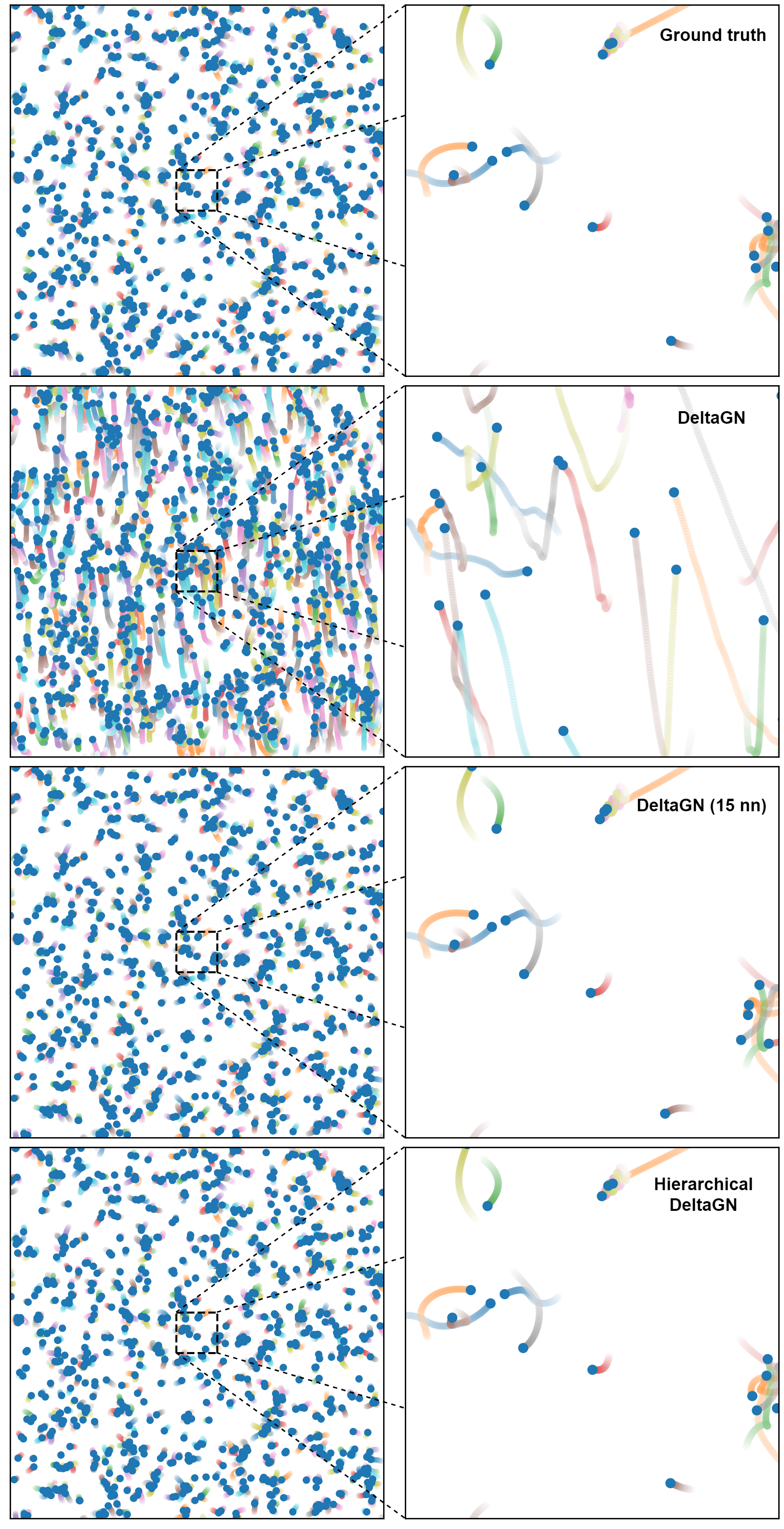}
    \caption{Sample simulation from the 1000 particle gravitational test set and the corresponding model predictions using a batch size of 1. On the right, we have a zoomed-in central region. In the first row, we have the ground truth simulation. In the second, third and fourth row we have the predictions produced by the best DeltaGN, DeltaGN (15 nn) and Hierarchical DeltaGN model runs respectively. This plot corresponds to the models trained in the Small Batch Size section.}
    \label{fig:1k_trajecotries_pred_batch_1}
\end{figure*}

\begin{figure*}[t]
    \centering
    \includegraphics[width=0.575\textwidth]{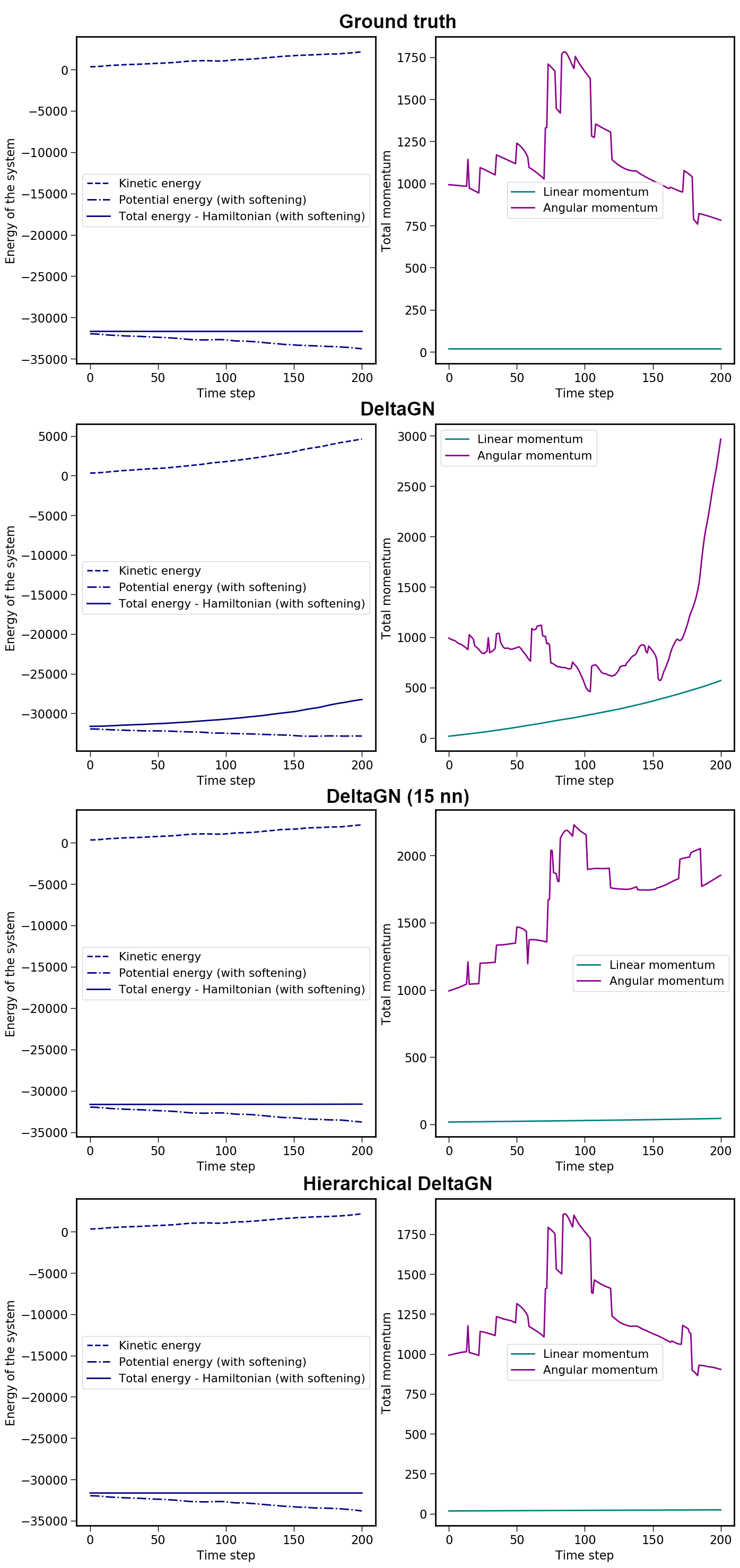}
    \caption{Energy (left column) and momentum (right column) plots corresponding to the sample simulation and predictions from Figure \ref{fig:1k_trajecotries_pred_batch_1}.}
    \label{fig:1k_trajecotries_pred_batch_1_energies}
\end{figure*}

\begin{figure*}[t]
    \centering
    \includegraphics[width=0.79\textwidth]{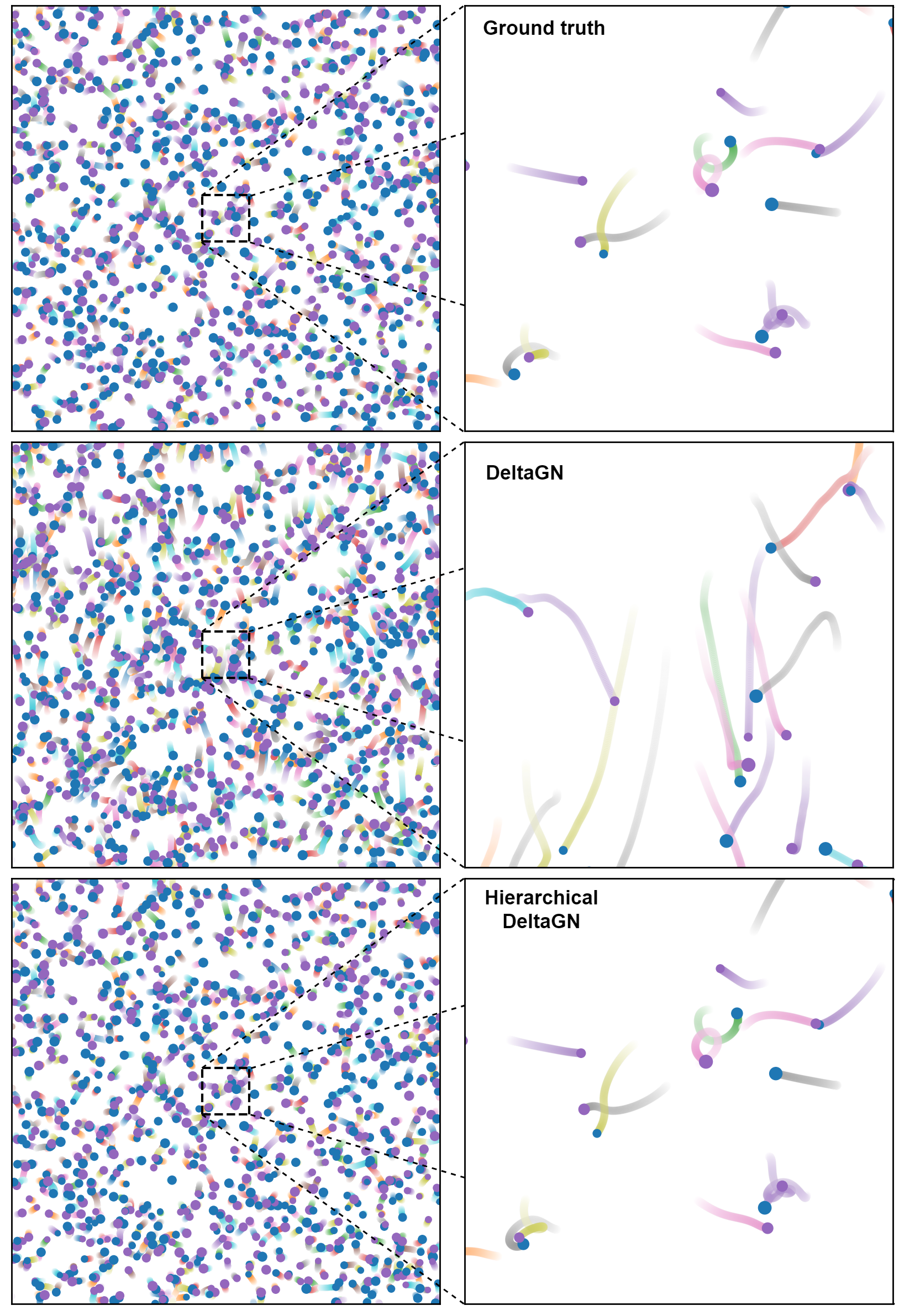}
    \caption{Sample simulation from the 1000 particle Coulomb test set and the corresponding predictions by the best runs of the DeltaGN and hierarchical DeltaGN models. DeltaGN used a batch size of 4 during training, while hierarchical DeltaGN used a batch size of 100. On the right, we have a zoomed-in central region. The particle size represents the magnitude of their charge. Blue particles are positively charged while violet particles are negatively charged. This plot corresponds to the models trained in the Coulomb Interactions section.}
    \label{fig:1k_coulomb}
\end{figure*}

\begin{figure*}[t]
    \centering
    \includegraphics[width=0.72\textwidth]{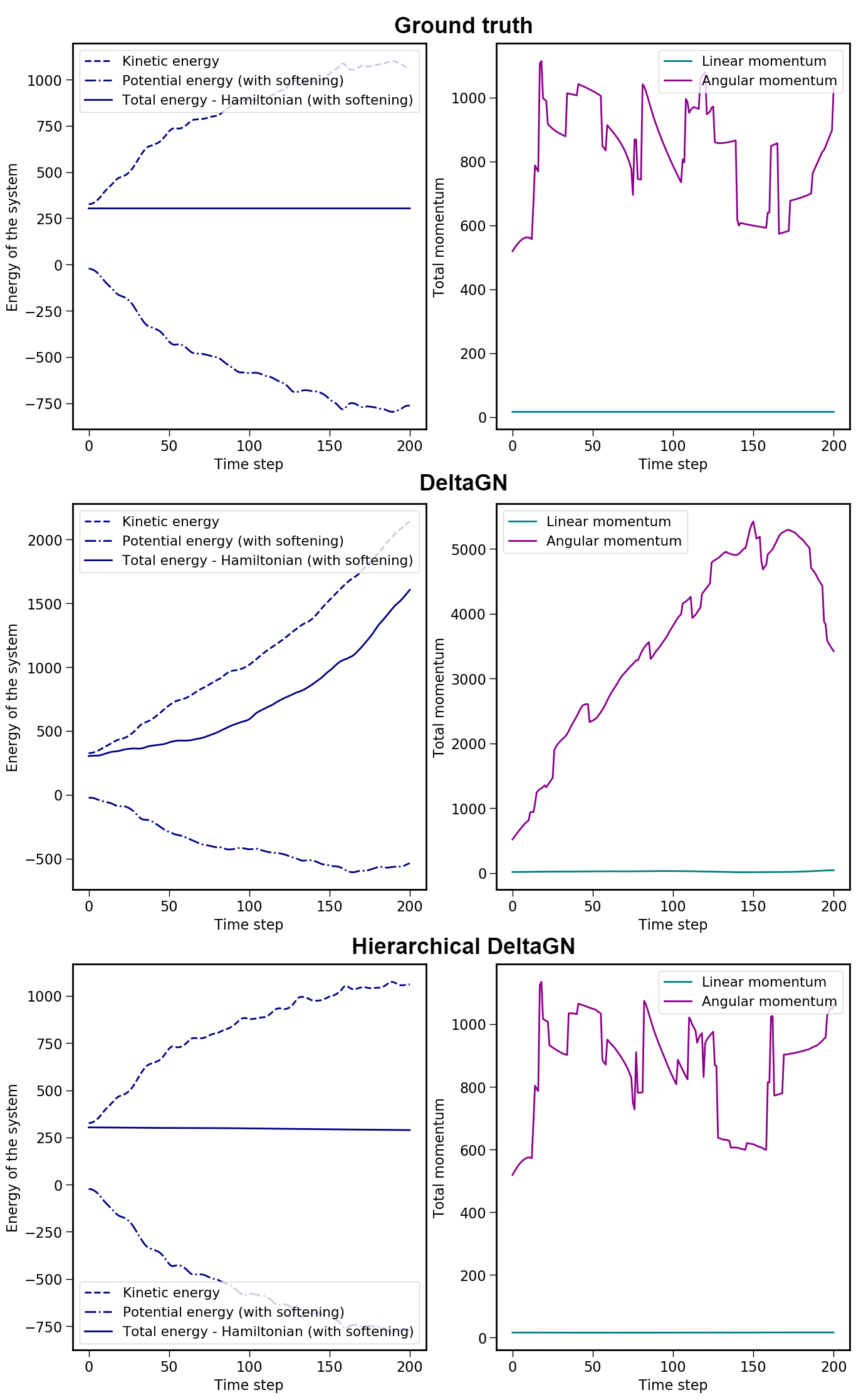}
    \caption{Energy (left column) and momentum (right column) plots corresponding to the sample simulation and predictions from Figure \ref{fig:1k_coulomb}.  This plot corresponds to the models trained in the Coulomb Interactions section.}
    \label{fig:1k_coulomb_energies}
\end{figure*}

\end{document}